\documentclass[a4paper]{amsart}

\usepackage{amsmath}
\usepackage{amsthm}
\usepackage{amssymb}
\usepackage{paralist}   
\usepackage[mathscr]{euscript}
\usepackage{bm}
\usepackage{hyperref}
\usepackage{url}
\usepackage{graphics,graphicx}
\usepackage{multirow}
\usepackage{algorithm}
\usepackage{algorithmic}
\usepackage{subfigure}
\usepackage{wrapfig}
\usepackage[labelfont=bf,tableposition=top]{caption}
\usepackage{lscape}
\usepackage{authblk}
\usepackage{amsaddr}
\usepackage{thmtools}
\usepackage{rotating}
\usepackage{tikz}
\usetikzlibrary{arrows}

\newtheorem{theorem}{Theorem}
\newtheorem{lemma}[theorem]{Lemma}

\newtheorem{proposition}{Proposition}

\declaretheoremstyle[notefont=\bfseries,notebraces={}{},%
    headpunct={},postheadspace=1em]{mystyle}
\declaretheorem[style=mystyle,numbered=no,name=Theorem]{thm-hand}
\declaretheorem[style=mystyle,numbered=no,name=Lemma]{lem-hand}

\newcommand{\inspace}{\ensuremath{\mathcal{X}}}   
\newcommand{\pp}[1]{\ensuremath{\mathbb{#1}}}     

\newcommand{\hbspace}{\ensuremath{\mathscr{H}}}   

\newcommand{\empmm}{\ensuremath{\widehat{\mu}}}

\newcommand{\rr}{\mathbb{R}} 		          
\newcommand{\ep}{\mathbb{E}}                      
\newcommand{\kmat}{\mathbf{K}}                  
\newcommand{\bvec}{\bm{\beta}}                  
\newcommand{\id}{\mathbf{I}}

\newcommand{\dd}{\, \mathrm{d}}

\hypersetup{
  colorlinks,
  citecolor=blue,
  linkcolor=blue,
  urlcolor=blue
}

\title{Kernel Mean Estimation and Stein's Effect}

\date{\today}

\author{Krikamol Muandet}
\address{Empirical Inference Department,  
  MPI for Intelligent Systems \\
  \texttt{\Large \href{mailto:krikamol@tuebingen.mpg.de}{krikamol@tuebingen.mpg.de}}}

\author{Kenji Fukumizu}
\address{The Institute of Statistical Mathematics \\
  \texttt{\Large \href{mailto:fukumizu@ism.ac.jp}{fukumizu@ism.ac.jp}}}

\author{Bharath Sriperumbudur}
\address{Statistical Laboratory, University of Cambridge \\
\texttt{\Large \href{mailto:bs493@statslab.cam.ac.uk}{bs493@statslab.cam.ac.uk}}}
  
\author{Arthur Gretton}
\address{Gatsby Computational Neuroscience Unit, University College
  London \\
  \texttt{\Large \href{mailto:arthur.gretton@gmail.com}{arthur.gretton@gmail.com}}}
  
\author{Bernhard Sch\"{o}lkopf}
\address{Empirical Inference Department, MPI for
  Intelligent Systems \\
  \texttt{\Large \href{mailto:bs@tuebingen.mpg.de}{bs@tuebingen.mpg.de}}}

\begin{document}

\maketitle

\begin{abstract}
  A mean function in reproducing kernel Hilbert space, or a kernel mean, is an important part of many applications ranging from kernel principal component analysis to Hilbert-space embedding of distributions. Given finite samples, an empirical average is the standard estimate for the true kernel mean. We show that this estimator can be improved via a well-known phenomenon in statistics called Stein's phenomenon. After consideration, our theoretical analysis reveals the existence of a wide class of estimators that are better than the standard. Focusing on a subset of this class, we propose efficient shrinkage estimators for the kernel mean. Empirical evaluations on several benchmark applications clearly demonstrate that the proposed estimators outperform the standard kernel mean estimator.
\end{abstract}

\section{Introduction}

This paper aims to improve the estimation of the mean function in a reproducing kernel Hilbert space (RKHS) from a finite number of samples. A kernel mean of a probability distribution $\pp{P}$ over a measurable space $\inspace$ is defined by
\begin{equation}
  \label{eq:kernel-mean}
  \mu_{\pp{P}} := \int_{\inspace}k(x,\cdot)\;\dd\pp{P}(x) \in
  \hbspace ,
\end{equation}  
\noindent where $\hbspace$ is an RKHS associated with a reproducing kernel $k:\inspace\times\inspace\rightarrow\rr$. Conditions ensuring that this expectation exists are given in \cite{Smola07Hilbert}. Unfortunately, it is not practical to compute $\mu_{\pp{P}}$ directly because the distribution $\pp{P}$ is usually unknown. Instead, given an i.i.d sample $x_1,x_2,\ldots,x_n$ from $\pp{P}$, we can easily compute the empirical kernel mean by the average
\begin{equation}
  \label{eq:empirical-mean}
  \widehat{\mu}_{\pp{P}} := \frac{1}{n}\sum_{i=1}^n k(x_i,\cdot)\,.
\end{equation}
The estimate $\widehat{\mu}_{\pp{P}}$ is the most natural and commonly used estimate of the true kernel mean. Our primary interest here is to investigate whether one can improve upon this standard estimator.


The mean function in RKHS serves as a foundation to many kernel-based algorithms. For instance, nonlinear component analyses, such as kernel PCA, kernel FDA, and kernel CCA, rely heavily on mean functions and covariance operators in RKHS \cite{Scholkopf98:NCA}. The kernel $K$-means algorithm performs clustering in feature space using mean functions as the representatives of the clusters \cite{Dhillon04:KKS}. All of those employ \eqref{eq:empirical-mean} as the estimate of the true mean function.


Recently the kernel mean itself has gained attention in the machine learning community. This is mainly due to the introduction of Hilbert space embedding for distributions \cite{Berlinet04:RKHS,Smola07Hilbert}. Representing the distribution as a mean function in the RKHS has several advantages: 1) the representation with appropriate choice of kernel $k$ has been shown to preserve all information about the distribution \cite{Fukumizu04:RKHS,Sriperumbudur08injectivehilbert,Sriperumbudur10:Metrics}; 2) basic operations on the distribution can be carried out by means of inner products in RKHS, e.g., $\ep_{\pp{P}}[f(x)]=\langle f,\mu_{\pp{P}}\rangle_{\hbspace}$ for all $f\in\hbspace$; 3) no intermediate density estimation is required, e.g., when testing for homogeneity from finite samples. As a result, there have been successful applications that benefit from the kernel mean representation, namely, a maximum mean discrepancy (MMD) \cite{Gretton07:MMD}, kernel dependency measure \cite{Gretton05:KIND}, kernel two-sample-test \cite{Gretton12:KTT}, Hilbert space embedding of HMMs \cite{Song10:HMM}, and kernel Bayes rule \cite{Fukumizu11:KBR}.


On the one hand, there are reasons to believe that the sample mean $\widehat{\mu}_{\pp{P}}$ should be optimal for estimation of the population mean $\mu_{\pp{P}}$. For example, it is the minimum-variance unbiased estimator (MVUE). Several discussions supporting this argument can be found in \cite{Lehmann98:point,Berlinet04:RKHS}. On the other hand, in 1955, Charles Stein \cite{Stein55:Inadmissible} showed that a maximum likelihood estimator (MLE), i.e., the standard empirical mean, for the mean of the multivariate Gaussian distribution $\mathcal{N}(\bm{\theta},\sigma^2\id)$ is inadmissible (see \S\ref{sec:inadmissibility} for a formal definition of admissibility). That is, there exists an estimator that always achieves smaller total mean squared error regardless of the true $\theta$, when the dimension is at least 3. Perhaps the best known estimator of such kind is James-Stein's estimator \cite{Stein61:JSE}. Interestingly, the James-Stein's estimator is itself inadmissible, and there exists a wide class of estimators that outperform the MLE, see e.g., \cite{Berger76:quadratic}. However, they all share a common feature: the mean estimation of at least one coordinate involves observations from other coordinates. Extensive research of Stein's result is available in, e.g., \cite{Lehmann98:point} and references therein.



As the kernel mean estimator \eqref{eq:empirical-mean} is similar in form to the MLE, i.e., the estimate of $\mu_{\pp{P}}$ is just the empirical average, one might suspect that it can be improved upon via Stein's phenomenon.\footnote{Though $\mu_{\pp{P}}$ may be viewed as an asymptotic mean of Gaussian measure on RKHS \cite[Theorem 108]{Berlinet04:RKHS}, whether we can regard $\widehat{\mu}_{\pp{P}}$ as a true maximum likelihood of some distribution on RKHS is still an open problem.} Although some attempts have been made to understand this phenomenon in the infinite-dimensional space \cite{Berger83:GP-Stein,Mandelbaum87:admissibility,Privault08:GP-malliavin}, this work presents a key challenge: the true probability distribution of $\mu_{\pp{P}}$ is essentially unknown and is subject to the choice of kernel $k$ and distribution $\pp{P}$, whereas all the previous works construct an estimator that outperforms a specific probability distribution. Despite this difference it is possible to establish the inadmissibility of the standard estimator \eqref{eq:empirical-mean} under certain assumptions, and construct better estimators.

The contribution of this paper can be summarised as follows: First, we show that the standard kernel mean estimator is inadmissible by providing an alternative estimator that achieves smaller expected loss (\S \ref{sec:inadmissibility}). This inadmissibility perspective sheds light on how one could construct better estimators. To this end, we propose a \emph{kernel mean shrinkage estimator} (KMSE) which relies on a fundamentally different framework from what has often been considered in the literature (\S\ref{sec:kmse}). Moreover, we propose an efficient leave-one-out cross-validation procedure to select the shrinkage parameter. Lastly, we demonstrate the benefit of the proposed estimators in several benchmark applications (\S\ref{sec:experiments}).



\section{Motivation: Standard Kernel Mean Estimator Is Inadmissible}
\label{sec:inadmissibility}

For fixed distribution $\pp{P}$, denote by $\mu$ and $\empmm$ the true kernel mean and its empirical estimate \eqref{eq:empirical-mean} from the sample $x_1,x_2,\ldots,x_n\sim\pp{P}$. We consider the loss function 
\begin{equation*}
\ell(\mu,\widehat{\mu}) = \|\mu - \widehat{\mu}\|^2_{\hbspace} . 
\end{equation*}
An estimator $\delta:\hbspace\rightarrow\hbspace$ is a mapping which is measurable w.r.t. the Borel $\sigma$-algebra of $\hbspace$. The estimator $\delta$ is evaluated by its risk function $\mathcal{R}(\mu,\delta) = \ep_{\pp{P}}[\ell(\mu,\delta(\phi(x)))]$. An estimator $\delta'$ is as good as $\delta$ if $\mathcal{R}(\mu,\delta') \leq \mathcal{R}(\mu,\delta)$, and is better than $\delta$ if it is as good as $\delta$ and $\mathcal{R}(\mu,\delta')<\mathcal{R}(\mu,\delta)$ for at least one $\pp{P}$. An estimator is \emph{inadmissible} if there exists a better estimator, and it is admissible otherwise.

Let us consider an alternative kernel mean estimator: 
\begin{equation*}
\widehat{\mu}_{\alpha} := \alpha f^* + (1-\alpha)\widehat{\mu} ,
\end{equation*}
\noindent where $0\leq \alpha < 1$ and $f^*\in\hbspace$. It is basically a shrinkage estimator that shrinks the standard estimator toward a function $f^*$ by an amount specified by $\alpha$. If $\alpha=0$, $\widehat{\mu}_{\alpha}$ reduces to the standard estimator. The following theorem asserts that the standard estimator $\widehat{\mu}$ is inadmissible w.r.t. the shrinkage estimator $\widehat{\mu}_{\alpha}$ with an appropriate choice of $\alpha$, regardless of the function $f^*$ (more below).
\begin{theorem}[Inadmissibility]
  \label{thm:inadmissibility}
  Given an i.i.d. sample $x_1,x_2,\ldots,x_n\sim\pp{P}$ and an arbitrary function $f^*\in\hbspace$, there exists $\alpha_*$ for which $\empmm_{\alpha_*}$ is better than $\empmm$.
\end{theorem}

\begin{proof}[Proof of Theorem \ref{thm:inadmissibility}]
  The risk of standard kernel mean estimator satisfies 
  \begin{equation*}
    \ep\|\empmm - \mu\|^2 = \frac{1}{n}\left(\ep[k(x,x)] - \ep[k(x,\tilde{x})]\right) =: \Delta
  \end{equation*}
  \noindent where $\tilde{x}$ is an independent copy of $x$. Let us define the risk of the proposed shrinkage estimator by $\Delta_\alpha := \ep\|\widehat{\mu}_{\alpha} - \mu\|^2$ where $0\leq \alpha < 1$. We can then write this in terms of the standard risk as 
  \begin{equation*}
  \Delta_{\alpha} = \Delta - 2\alpha\ep\left\langle \empmm-\mu,\empmm-\mu+\mu-f^*\right\rangle + \alpha^2\ep\|f^*\|^2 - 2\alpha^2\ep[f^*(x)] + \alpha^2\ep\|\empmm\|^2.
\end{equation*}
  It follows from the reproducing property of $\hbspace$ that
  $\ep[f^*(x)] = \langle f^*,\mu \rangle$. Moreover, using the fact that
  $\ep\|\empmm\|^2 = \ep\|\empmm - \mu + \mu\|^2 = \Delta +
  \ep[k(x,\tilde{x})]$, we can simplify the shrinkage risk by 
  \begin{equation*}
  \Delta_\alpha = \left(\Delta + \|f^* - \mu\|^2\right)\left(\alpha -
    \Delta/(\Delta + \|f^*-\mu\|^2)\right)^2 + \Delta - (\Delta^2/(\Delta + \|f^* - \mu\|^2) .
  \end{equation*}
  Then, the shrinkage parameter given by $\alpha_* := \Delta/(\Delta + \|f^* - \mu\|^2)$ leads to $\Delta_{\alpha_*} - \Delta = - \Delta^2/(\Delta + \|f^* -
\mu\|^2) \leq 0$. The equality occurs only when $\Delta = 0$.
\end{proof}

    
    
    



Theorem \ref{thm:inadmissibility} relies on important assumption that the true kernel mean of the distribution $\pp{P}$ is required to estimate $\alpha_*$. In spite of this, the theorem has an important implication suggesting that the shrinkage estimator $\widehat{\mu}_{\alpha}$ can improve upon $\widehat{\mu}$ if $\alpha$ is chosen appropriately. In the next section, we will exploit this result in order to contruct more practical estimators. Moreover, it is striking to see that the shrinkage estimator always improves upon the standard one regardless of the direction of shrinkage, as specified by the choice of $f^*$. In other words, there exists a wide class of kernel mean estimators that are better than the standard one. Inspired by  James-Stein's estimator, we will focus on $f^*=\mathbf{0}$ in the following sections.


\section{Kernel Mean Shrinkage Estimator}
\label{sec:kmse}

Since the choice of $\alpha$ is very crucial, we propose a new formulation of kernel mean estimator that will allow us to estimate $\alpha$ systematically and efficiently. Let $\phi:\inspace\rightarrow\hbspace$ be a feature map associated with the kernel $k$ and $\langle \cdot,\cdot\rangle$ be an inner product in the RKHS $\hbspace$ such that $k(x,x')=\langle\phi(x),\phi(x')\rangle$. The kernel mean $\mu_{\pp{P}}$ and its empirical estimate $\widehat{\mu}_{\pp{P}}$ can then be obtained as a minimizer of the loss functionals
\begin{equation*} 
  \mathcal{E}(g) := \ep_{x\sim\pp{P}}\left\|\phi(x) -
    g\right\|^2_{\hbspace}\; \text{ and } \; \widehat{\mathcal{E}}(g) := \frac{1}{n}\sum_{i=1}^n\left\|\phi(x_i)
    - g\right\|^2_{\hbspace},
\end{equation*}
\noindent respectively. Let $\delta(\phi(x)) :=
\arg\inf_{g\in\hbspace}\widehat{\mathcal{E}}(g) = \phi(x)$ be the
estimator associated with the loss functional
$\widehat{\mathcal{E}}(g)$. In the following, we call this standard
estimator a \textbf{kernel mean estimator (KME)}.

Note that the loss $\mathcal{E}(g)$ is different from the one we use in \S\ref{sec:inadmissibility}, i.e., $\ell(\mu,g)=\|\mu - g\|_{\hbspace}^2=\|\ep[\phi(x)]-g\|_{\hbspace}^2$. However, by Jensen's inequality, we have $\|\ep[\phi(x)]-g\|_{\hbspace}^2 \leq \ep\|\phi(x)-g\|_{\hbspace}^2 =: \mathcal{E}(g)$. Hence, both functionals have a minimum at the same $g$. In addition, the new form will give a more tractable leave-one-out cross-validation computation (\S\ref{sec:cross-validation}).

To construct the shrinkage estimator, we minimize a modified loss functional
\begin{equation}
  \label{eq:empirical-loss}
  \widehat{\mathcal{E}}_{\lambda}(g) := \widehat{\mathcal{E}}(g) + \lambda\Omega(\|g\|) = \frac{1}{n}\sum_{i=1}^n\left\|\phi(x_i)
    - g\right\|^2_{\hbspace} + \lambda\Omega(\|g\|),
\end{equation}
\noindent where $\Omega(\cdot)$ denotes a monotonically-increasing shrinkage functional and $\lambda$ is a non-negative shrinkage parameter. In what follows, we refer to the shrinkage estimator $\delta_{\lambda}(\cdot)$  associated with $\widehat{\mathcal{E}}_{\lambda}(\bvec)$ as a \textbf{kernel mean shrinkage estimator (KMSE)}. It follows from the representer theorem that $g$ lies in a subspace spanned by the data, i.e., $g=\sum_{j=1}^n\beta_j\phi(x_j)$ for some $\bvec\in\rr^n$. Firstly, by considering $\Omega(\|g\|)=\|g\|^2$, we can rewrite \eqref{eq:empirical-loss} as
\begin{eqnarray}
  \label{eq:shrinkage-loss}
  \widehat{\mathcal{E}}_{\lambda}(g) &=& \frac{1}{n}\sum_{i=1}^n\left\|\phi(x_i)
    - \sum_{j=1}^n\beta_j\phi(x_j) \right\|^2_{\hbspace} +
  \lambda\left\|\sum_{j=1}^n\beta_j\phi(x_j)\right\|^2_{\hbspace}
  \nonumber \\
  &=& \bvec^\top \kmat\bvec -
  2\bvec^\top \kmat\mathbf{1}_n + \lambda\bvec^\top\kmat\bvec + c,
\end{eqnarray} 
\noindent where $c$ is a constant term, $\kmat$ is an $n\times n$ gram matrix such that $\kmat_{ij} = k(x_i,x_j)$, and $\mathbf{1}_n = [1/n,1/n,\ldots,1/n]^\top$. Taking a derivative of \eqref{eq:shrinkage-loss} w.r.t. $\bvec$ and setting it to zero yield 
\begin{equation*}
  \bvec = (1/(1+\lambda))\mathbf{1}_n. 
\end{equation*}
By setting $\alpha = \lambda/(1+\lambda)$ the shrinkage estimate can be written as $\widehat{\mu}_{\lambda} := \delta_{\lambda}(\widehat{\mu}) = (1-\alpha)\widehat{\mu}$. Since $1-\alpha \leq 1$, the estimator $\delta_{\lambda}(\cdot)$ corresponds to a shrinkage estimator discussed in \S\ref{sec:inadmissibility} when $f^*=\mathbf{0}$. We will call this estimator a \textbf{simple kernel mean shrinkage estimator (S-KMSE)}.

Another interesting choice of shrinkage functional is $\Omega(\|g\|) = \bvec^\top\bvec$. This leads to a particularly interesting kernel mean estimator. In this case, the optimal weight vector is given by 
\begin{equation*}
  \bvec = (\kmat+\lambda\id)^{-1}\kmat\mathbf{1}_n
\end{equation*}
and the shrinkage estimate can be written accordingly as 
\begin{equation*}
  \widehat{\mu}_{\lambda} := \delta_{\lambda}(\widehat{\mu}) = \sum_{j=1}^n\beta_j\phi(x_j) = \Phi^\top(\kmat+\lambda\id)^{-1}\kmat\mathbf{1}_n
\end{equation*}
\noindent where $\Phi = [\phi(x_1),\phi(x_2),\ldots,\phi(x_n)]^\top$. Unlike the S-KMSE, this estimator shrinks the usual estimate differently in each coordinate (cf. Theorem \ref{thm:akmse-shrinkage}). Hence, we will call it a \textbf{flexible kernel mean shrinkage estimator (F-KMSE)}.

Notice that our formulation differs fundamentally from the standard regularization framework. That is, the shrinkage functional $\Omega(\cdot)$ is introduced to shrink the solution $g$ toward certain point, i.e., $f^*$, rather than to regularize it. As the original problem is well-posed and its solution can be computed analytically, regularization is not necessary in this context.


As we can see, both S-KMSE and F-KMSE shrink the kernel mean estimate towards zero, but the F-KMSE does so in a more flexible way, as shown in the following theorem.

\begin{theorem}
  \label{thm:akmse-shrinkage}
  For F-KMSE, we can write $\widehat{\mu}_{\lambda} = \sum_{i=1}^n
  \frac{\gamma_i}{\gamma_i+\lambda}\langle
  \widehat{\mu},\mathbf{v}_i\rangle\mathbf{v}_i$ where
  $\{\gamma_i,\mathbf{v}_i\}$ are eigenvalue and eigenvector pairs of
  the covariance operator $\widehat{\mathbf{C}}_{xx}$ in $\hbspace$.
\end{theorem}

\begin{proof}[Proof of Theorem \ref{thm:akmse-shrinkage}]
Assume that we know the eigendecomposition $\kmat=\mathbf{UDU}^{\top}$ where $\mathbf{U} = [\mathbf{u}_1,\mathbf{u}_2,\ldots,\mathbf{u}_n]$ consists of orthogonal eigenvectors of $\kmat$ such that $\mathbf{U}^{\top}\mathbf{U} = \mathbf{I}$ and $\mathbf{D}=\mathrm{diag}(\gamma_1,\gamma_2,\ldots,\gamma_n)$ consists of corresponding eigenvalues. Hence, the weights $\bvec$ of the F-KMSE is given by
\begin{equation*}
  \bvec = (\mathbf{UDU}^{\top} + \lambda\mathbf{I})^{-1}\kmat\mathbf{1}_n = (\mathbf{U}(\mathbf{D}+\lambda\mathbf{I})\mathbf{U}^{\top})^{-1}\kmat\mathbf{1}_n = \mathbf{U}(\mathbf{D}+\lambda\mathbf{I})^{-1}\mathbf{U}^{\top}\kmat\mathbf{1}_n .
\end{equation*}
Consequently, 
\begin{equation}
  \label{eq:wf-kmse}
\bvec = \sum_{i=1}^n\mathbf{u}_i\left(\frac{1}{\gamma_i+\lambda}\right)\mathbf{u}_i^{\top}\kmat\mathbf{1}_n .
\end{equation}
Note also that
\begin{equation*}
  \kmat\mathbf{1}_n = \left[\frac{1}{n}\sum_{j=1}^nk(x_j,x_1),\ldots,\frac{1}{n}\sum_{j=1}^nk(x_j,x_n)\right]^\top = \left[\langle\widehat{\mu},\phi(x_1)\rangle,\ldots,\langle\widehat{\mu},\phi(x_n)\rangle\right]^\top .
\end{equation*}
Thus, we can rewrite \eqref{eq:wf-kmse} as
\begin{eqnarray*}
  \bvec &=& \sum_{i=1}^n\mathbf{u}_i\left(\frac{1}{\gamma_i+\lambda}\right) \sum_{j=1}^n u_{ij}\langle\widehat{\mu},\phi(x_j)\rangle \\
  &=& \sum_{i=1}^n\mathbf{u}_i\left(\frac{\sqrt{\gamma_i}}{\gamma_i+\lambda}\right) \left\langle\widehat{\mu},\frac{1}{\sqrt{\gamma_i}}\sum_{j=1}^n u_{ij}\phi(x_j)\right\rangle
\end{eqnarray*}
It follows from the correspondance between the eigenvectors of kernel
matrix $\kmat$ and covariance matrix $\widehat{\mathbf{C}}_{xx}$ that
$\mathbf{v}_i = (1/\sqrt{\gamma_i})\sum_ju_{ij}\phi(x_j)$ where
$\mathbf{v}_i$ is the $i$th eigenvector of the covariance matrix. Consequently, we have 
\begin{equation}
  \label{eq:meanmap-proj}
  \left\langle\widehat{\mu},\frac{1}{\sqrt{\gamma_i}}\sum_{j=1}^n u_{ij}\phi(x_j)\right\rangle = \left\langle \widehat{\mu},\mathbf{v}_i \right\rangle
\end{equation}
In words, \eqref{eq:meanmap-proj} is a projection of the standard
kernel mean embedding onto the eigenvector $\mathbf{v}_i$. Using this
representation, the shrinkage estimate of the F-KMSE given by the weights $\bvec$ becomes
\begin{equation*}
  \widehat{\mu}_{\lambda} = \sum_{j=1}^n\left[ \sum_{i=1}^n\mathbf{u}_i\left(\frac{\sqrt{\gamma_i}}{\gamma_i+\lambda}\right) \left\langle \widehat{\mu},\mathbf{v}_i \right\rangle\right]_j\phi(x_j) .
\end{equation*} 

Applying the same trick, we can write the F-KMSE estimate entirely in
term of eigenvectors of the covariance matrix $\widehat{\mathbf{C}}_{xx}$ as
\begin{eqnarray*}
  \widehat{\mu}_{\lambda} &=& \sum_{j=1}^n\phi(x_j)\sum_{i=1}^n u_{ij}\left(\frac{\sqrt{\gamma_i}}{\gamma_i+\lambda}\right) \left\langle \widehat{\mu},\mathbf{v}_i \right\rangle \\
  &=& \sum_{i=1}^n \left(\frac{\sqrt{\gamma_i}}{\gamma_i+\lambda}\right) \left\langle \widehat{\mu},\mathbf{v}_i \right\rangle\sum_{j=1}^n u_{ij}\phi(x_j) \\
  &=& \sum_{i=1}^n \left(\frac{\gamma_i}{\gamma_i+\lambda}\right) \left\langle \widehat{\mu},\mathbf{v}_i \right\rangle\mathbf{v}_i
\end{eqnarray*}
Since $\lambda > 0$, we have that $\gamma_i/(\gamma_i + \lambda) <
1$. This completes the proof.
\end{proof}

In words, the effect of F-KMSE shrinkage is related to the sample variance in feature space, which can be viewed as the amount of information available in each coordinate. To be more precise, the F-KMSE restricts the shrinkage in directions with high variance and allows more shrinkage in low-variance directions.


Moreover, the squared RKHS norm $\|\cdot\|^2_{\hbspace}$ can be decomposed as a sum of squared loss weighted by the eigenvalues $\gamma_i$ (cf. \cite[Appendix]{Mandelbaum87:admissibility}). By the same reasoning as Stein's result in finite-dimensional case, one would suspect that an improvement of shrinkage estimators in $\hbspace$ should also depend on how fast the eigenvalues of $k$ decay. That is, one would expect greater improvement if the values of $\gamma_i$ decay very slowly. For example, the Gaussian RBF kernel with larger bandwidth gives smaller improvement when compared to one with smaller bandwidth. Similarly, we should expect to see more improvement when applying a Laplacian kernel than when using a Gaussian RBF kernel.

The weight vector $\bvec$ output by our estimators is in general not normalized. In fact, all elements will be smaller than $1/n$ as a result of shrinkage. However, one may impose a constraint that $\bvec$ must sum to one and resort to a quadratic programming \cite{Song08:TDE}. Unfortunately, this approach has undesirable effect of sparsity which is unlikely to improve upon the standard estimator. Post-normalizing the weights often deteriorates the estimation performance.

Recently, attempts have been made to improve the kernel mean estimation in various contexts. In \cite{Kim12:RKDE}, the loss functional $\widehat{\mathcal{E}}(g)$ is replaced by a robust loss function such as the Huber's loss to reduce the effect of outliers. Regularized version of MMD was adopted by \cite{Danafar13:RMMD} in the context of kernel-based hypothesis testing. The resulted formulation resembles our S-KMSE. Furthermore, the F-KMSE is of a similar form as the conditional mean embedding used in \cite{Grunewalder12:LGBPP}, which can be viewed more generally as a regression problem in RKHS with smooth operators \cite{Grunewalder13:SO}. Despite this similarity, one should note that in this work we treat the problem entirely as estimation problem, and hence it is fundamentally different from the existing works.


\subsection{Cross-validation}
\label{sec:cross-validation}

As discussed in \S\ref{sec:inadmissibility}, the amount of shrinkage plays an important role in our estimators. In this work we propose to select the shrinkage parameter $\lambda$ by an automatic leave-one-out cross-validation.

For a given shrinkage parameter $\lambda$, let us consider the observation $x_i$ as being a new observation by omitting it from the dataset. Denote by $\widehat{\mu}^{(-i)}_{\lambda} = \sum_{j\neq i}\beta_j^{(-i)}\phi(x_j)$ the kernel mean estimated from the remaining data, using the value $\lambda$ as a shrinkage parameter, so that $\bvec^{(-i)}$ is the minimizer of $\widehat{\mathcal{E}}^{(-i)}_{\lambda}(g)$. We will measure the quality of $\widehat{\mu}^{(-i)}_{\lambda}$ by how well it approximates $\phi(x_i)$. The overall quality of the estimate is quantified by the cross-validation score
\begin{equation}
  \label{eq:loocv-score}
  LOOCV(\lambda) = \frac{1}{n}\sum_{i=1}^n\left\| \phi(x_i) - \widehat{\mu}_{\lambda}^{(-i)}\right\|^2_{\hbspace}.
\end{equation}
By simple algebra, it is not difficult to show that the optimal shrinkage parameter of S-KMSE can be calculated analytically, as stated by the following theorem.


\begin{theorem}
  \label{thm:skmse-loocv}
  Let $\rho:=\frac{1}{n^2}\sum_{i,j=1}^n k(x_i,x_j)$ and $\varrho:=\frac{1}{n}\sum_{i=1}^n k(x_i,x_i)$. The shrinkage parameter $\lambda_* = (\varrho-\rho)/((n-1)\rho + \varrho/n-\varrho)$ of the S-KMSE is the minimizer of $LOOCV(\lambda)$.
\end{theorem}

\begin{proof}[Proof of Theorem \ref{thm:skmse-loocv}]
  Note that the leave-one-out cross-validation score for the S-KMSE is
\begin{equation*}
  LOOCV(\alpha) := \frac{1}{n}\sum_{i=1}^n\left\|(1-\alpha)\widehat{\mu}_\lambda^{(-i)} - \phi(x_i)\right\|^2_{\hbspace},
\end{equation*}
\noindent which can be simplified further as
\begin{eqnarray*}
  LOOCV(\alpha) &=& \frac{1}{n}\sum_{i=1}^n\left\|\frac{n}{n-1}(1-\alpha)\widehat{\mu} - \frac{1-\alpha}{n-1}\phi(x_i) - \phi(x_i)\right\|^2_{\hbspace} \\
  &=& \left\|\frac{n}{n-1}(1-\alpha)\widehat{\mu}\right\|^2_{\hbspace}
  -
  \frac{2}{n}\left\langle\sum_{i=1}^n\frac{n-\alpha}{n-1}\phi(x_i),\frac{n}{n-1}(1-\alpha)\widehat{\mu}\right\rangle
  \\
  && + \frac{1}{n}\sum_{i=1}^n\left\|\frac{n-\alpha}{n-1}\phi(x_i)\right\|^2_{\hbspace} \\
  &=& \frac{n^2(1-\alpha)^2}{(n-1)^2}\|\widehat{\mu}\|^2 - \left(\frac{2}{n}\right)\left(\frac{(n-\alpha)n}{n-1}\right)\left(\frac{n(1-\alpha)}{n-1}\right)\|\widehat{\mu}\|^2 \\
  && + \frac{1}{n}\left(\frac{n-\alpha}{n-1}\right)^2\sum_{i=1}^n k(x_i,x_i) \\
  &=& \left(\frac{n^2(1-\alpha)^2}{(n-1)^2} - \frac{2n(n-\alpha)(1-\alpha)}{(n-1)^2}\right)\|\widehat{\mu}\|^2 \\
  && + \frac{(n-\alpha)^2}{n(n-1)^2}\sum_{i=1}^n k(x_i,x_i)
\end{eqnarray*}

Let $\rho:=\frac{1}{n^2}\sum_{i,j=1}^nk(x_i,x_j)$ and $\varrho:=\frac{1}{n}\sum_{i=1}^nk(x_i,x_i)$. Then, the leave-one-out score becomes
\begin{eqnarray*}
  LOOCV(\alpha) &=& \frac{1}{(n-1)^2}\left\{(-n^2 + \alpha^2n^2 + 2\alpha n - 2\alpha^2 n)\rho + (n^2-2\alpha n + \alpha^2)\varrho\right\} 
\end{eqnarray*} 
Taking the derivative of $LOOCV(\alpha)$ with respect to $\alpha$ and setting it to zero yield
\begin{equation*}
  \alpha_* = \frac{\varrho - \rho}{(n-2)\rho + \varrho/n},
\end{equation*}
Since the parameter $\alpha$ is given by $\alpha =
\lambda/(1+\lambda)$, it follows that 
\begin{equation*}
  \lambda_* = \frac{\varrho - \rho}{(n - 1)\rho + \varrho/n - \varrho}
\end{equation*}
\noindent as required.
\end{proof}

However, finding the optimal $\lambda$ for the F-KMSE is relatively
more involved. Evaluating the score  \eqref{eq:loocv-score}
na\"{\i}vely requires one to solve for
$\widehat{\mu}_{\lambda}^{(-i)}$ explicitly for every
$i$. Fortunately, we can simplify the score such that it can be
evaluated efficiently, as stated in the following theorem (see the
appendix for the detailed proof).


\begin{theorem}
  \label{thm:akmse-loocv}
  The LOOCV score of F-KMSE satisfies 
  \begin{equation*}
  LOOCV(\lambda) = \frac{1}{n}\sum_{i=1}^n (\bvec^{\top}\kmat -
  \kmat_i)^\top \mathbf{C}_{\lambda}(\bvec^{\top}\kmat - \kmat_i)
  \end{equation*}
  \noindent where $\bvec$ is the weight vector calculated from the full dataset with the shrinkage parameter $\lambda$ and $\mathbf{C}_{\lambda} = (\kmat - \frac{1}{n}\kmat(\kmat+\lambda\id)^{-1}\kmat)^{-1} \kmat (\kmat - \frac{1}{n}\kmat(\kmat + \lambda \id)^{-1}\kmat)^{-1}$.
\end{theorem}

\begin{proof}[Proof of Thorem \ref{thm:akmse-loocv}]
  For fixed $\lambda$ and $i$, let $\widehat{\mu}^{(-i)}_\lambda$ be
  the leave-one-out kernel mean estimate of F-KMSE and let $\mathbf{A}
  := (\kmat + \lambda\id)^{-1}$. Then, we can write an expression for
  the deleted residual as 
  \begin{equation*}
    \Delta^{(-i)}_{\lambda} := \widehat{\mu}^{(-i)}_{\lambda}-\phi(x_i) =
    \widehat{\mu}_{\lambda} - \phi(x_i) +
    \frac{1}{n}\sum_{j=1}^n\sum_{l=1}^n\mathbf{A}_{jl}\langle\phi(x_l),\widehat{\mu}^{(-i)}_{\lambda}-\phi(x_i)\rangle\phi(x_j). 
\end{equation*}
Since $\Delta^{(-i)}_{\lambda}$ lies in a subspace spanned by the
sample $\phi(x_1),\ldots,\phi(x_n)$, we have $\Delta^{(-i)}_{\lambda}
= \sum_{k=1}^n\xi_k\phi(x_k)$ for some
$\bm{\xi}\in\rr^n$. Substituting $\Delta^{(-i)}_{\lambda}$ back yields 
\begin{equation*}
\sum_{k=1}^n\xi_k\phi(x_k) = \widehat{\mu}_{\lambda} - \phi(x_i) +
\frac{1}{n}\sum_{j=1}^n\{\mathbf{A}\kmat\bm{\xi}\}_j\phi(x_j) . 
\end{equation*}
By taking the inner product on both sides w.r.t. the sample
$\phi(x_1),\ldots,\phi(x_n)$ and solving for $\bm{\xi}$, we have
$\bm{\xi} = (\kmat -
\frac{1}{n}\kmat\mathbf{A}\kmat)^{-1}(\bvec^\top\kmat - \kmat_{\cdot
  i})$ where $\kmat_{\cdot i}$ is the $i$th column of
$\kmat$. Consequently, the leave-one-out score of the sample $x_i$ can
be computed by 
\begin{eqnarray*}
\|\Delta^{(-i)}_{\lambda}\|^2 &=& \bm{\xi}^\top\kmat\bm{\xi} \\ 
&=& (\bvec^\top\kmat - \kmat_{\cdot i})^\top(\kmat -
\frac{1}{n}\kmat\mathbf{A}\kmat)^{-1}\kmat(\kmat -
\frac{1}{n}\kmat\mathbf{A}\kmat)^{-1}(\bvec^\top\kmat - \kmat_{\cdot
  i}) \\ 
&=& (\bvec^\top\kmat - \kmat_{\cdot
  i})^\top\mathbf{C}_{\lambda}(\bvec^\top\kmat - \kmat_{\cdot i}) . 
\end{eqnarray*}
Averaging $\|\Delta^{(-i)}_{\lambda}\|^2$ over all samples gives 
\begin{equation*}
LOOCV(\lambda)=\frac{1}{n}\sum_{i=1}^n \|\Delta^{(-i)}_{\lambda}\|^2
= \frac{1}{n}\sum_{i=1}^n(\bvec^\top\kmat - \kmat_{\cdot
  i})^\top\mathbf{C}_{\lambda}(\bvec^\top\kmat - \kmat_{\cdot i}) , 
\end{equation*}
as required.  
\end{proof}

It is interesting to see that the leave-one-out cross-validation score in Theorem \ref{thm:akmse-loocv} depends only on the non-leave-one-out solution $\bvec_{\lambda}$, which can be obtained as a by-product of the algorithm.

\subsubsection{Computational complexity}
The S-KMSE requires only $\mathcal{O}(n)$ operations to select shrinkage parameter. For the F-KMSE, there are two steps in cross-validation. First, we need to compute $(\kmat + \lambda\id)^{-1}$ repeatedly for different values of $\lambda$. Assume that we know the eigendecomposition $\kmat=\mathbf{UDU}^\top$ where $\mathbf{D}$ is diagonal with $d_{ii} \geq 0$ and $\mathbf{UU}^\top=\id$. It follows that $(\kmat + \lambda \id)^{-1}= \mathbf{U}(\mathbf{D}+\lambda \id)^{-1}\mathbf{U}^\top$. Consequently,  solving for $\bvec_\lambda$ takes $\mathcal{O}(n^2)$ operations. Since eigendecomposition requires $\mathcal{O}(n^3)$ operations, finding $\bvec_\lambda$ for many $\lambda$'s is essentially free. A low-rank approximation can also be adopted to reduce the computational cost further.

Second, we need to compute the cross-validation score \eqref{eq:loocv-score}. As shown in Theorem \ref{thm:akmse-loocv}, we can compute it using only $\bvec_\lambda$ obtained from the previous step. The calculation of $\mathbf{C}_{\lambda}$ can be simplified further via the eigendecomposition of $\kmat$ as 
\begin{equation*}
\mathbf{C}_{\lambda} = \mathbf{U}(\mathbf{D} - \frac{1}{n}\mathbf{D}(\mathbf{D}+\lambda \id)^{-1}\mathbf{D})^{-1}\mathbf{D}(\mathbf{D} - \frac{1}{n}\mathbf{D}(\mathbf{D}+\lambda \id)^{-1}\mathbf{D})^{-1}\mathbf{U}^\top .
\end{equation*} 
Since it only involves the inverse of diagonal matrices, the inversion can be evaluated in $\mathcal{O}(n)$ operations. The overall computational complexity of the cross-validation requires only $\mathcal{O}(n^2)$ operations, as opposed to the na\"{\i}ve approach that requires $\mathcal{O}(n^4)$ operations. When performed as a by-product of the algorithm, the computational cost of cross-validation procedure becomes negligible as the dataset becomes larger. In practice, we use the \texttt{fminsearch} and \texttt{fminbnd} routines of the MATLAB optimization toolbox to find the best shrinkage parameter. 


\subsection{Covariance Operators}

The covariance operator from $\hbspace_X$ to $\hbspace_Y$ can be viewed as a mean function in a product space $\hbspace_X\otimes\hbspace_Y$. Hence, we can also construct a shrinkage estimator of covariance operator in RKHS. Let $(\hbspace_X,k_X)$ and $(\hbspace_Y,k_Y)$ be the RKHS of functions on measurable space $\inspace$ and $\mathcal{Y}$, respectively, with p.d. kernel $k_X$ and $k_Y$ (with feature map $\phi$ and $\varphi$). We will consider a random vector $(X,Y):\Omega \rightarrow \inspace\times\mathcal{Y}$ with distribution $\pp{P}_{XY}$, with $\pp{P}_X$ and $\pp{P}_Y$ as marginal distributions. Under some conditions, there exists a unique cross-covariance operator $\Sigma_{YX}:\hbspace_X\rightarrow\hbspace_Y$ such that 
\begin{equation*}
  \langle g,\Sigma_{YX}f\rangle_{\hbspace_Y} = \ep_{XY}[(f(X) - \ep_X[f(X)])(g(Y) - \ep_Y[g(Y)])] = Cov(f(X),g(Y))
\end{equation*} 
\noindent holds for all $f\in\hbspace_X$ and $g\in\hbspace_Y$ \cite{Fukumizu04:RKHS}. If $X$ equals $Y$, we get the self-adjoint operator $\Sigma_{XX}$ called the covariance operator.

Given an i.i.d sample from $\pp{P}_{XY}$ written as $(x_1,y_1),(x_2,y_2),\ldots,(x_n,y_n)$, we can write the empirical cross-covariance operator as 
\begin{equation*}
\widehat{\Sigma}_{YX} := \frac{1}{n}\sum_{i=1}^n\phi(x_i)\otimes\varphi(y_i) - \widehat{\mu}_X\otimes\widehat{\mu}_Y ,
\end{equation*}
\noindent where $\widehat{\mu}_X = \frac{1}{n}\sum_{i=1}^n\phi(x_i)$ and $\widehat{\mu}_Y = \frac{1}{n}\sum_{i=1}^n\varphi(y_i)$. Let $\widetilde{\phi}$ and $\widetilde{\varphi}$ be the centered feature maps of $\phi$ and $\varphi$, respectively. Then, it can be rewritten as 
\begin{equation*}
\widehat{\Sigma}_{YX} := \frac{1}{n}\sum_{i=1}^n\widetilde{\phi}(x_i)\otimes\widetilde{\varphi}(y_i) \in\hbspace_X\otimes\hbspace_Y . 
\end{equation*}
It follows from the inner product property in product space that 
\begin{eqnarray*}
\langle \widetilde{\phi}(x)\otimes\widetilde{\varphi}(y),\widetilde{\phi}(x')\otimes\widetilde{\varphi}(y')\rangle_{\hbspace_X\otimes\hbspace_Y} &=& \langle\widetilde{\phi}(x),\widetilde{\phi}(x')\rangle_{\hbspace_X}\langle\widetilde{\varphi}(y),\widetilde{\varphi}(y')\rangle_{\hbspace_Y}\\
&=& \widetilde{k}_X(x,x')\widetilde{k}_Y(y,y') . 
\end{eqnarray*}
Then, we can obtain the shrinkage estimators for the covariance operator by plugging the kernel $k((x,y),(x',y')) = \tilde{k}_X(x,x')\tilde{k}_Y(y,y')$ in our KMSEs. We will call this estimator a \textbf{covariance-operator shrinkage estimator (COSE)}.

\section{Experiments}
\label{sec:experiments}

We focus on the comparison between our shrinkage estimators and the standard estimator of the kernel mean using both synthetic datasets and real-world datasets.

\subsection{Synthetic Data}

\begin{figure}[t!]
  \centering
  \subfigure[$\lambda = 0.01\times\gamma_0$]{
    \includegraphics[width=0.45\linewidth]{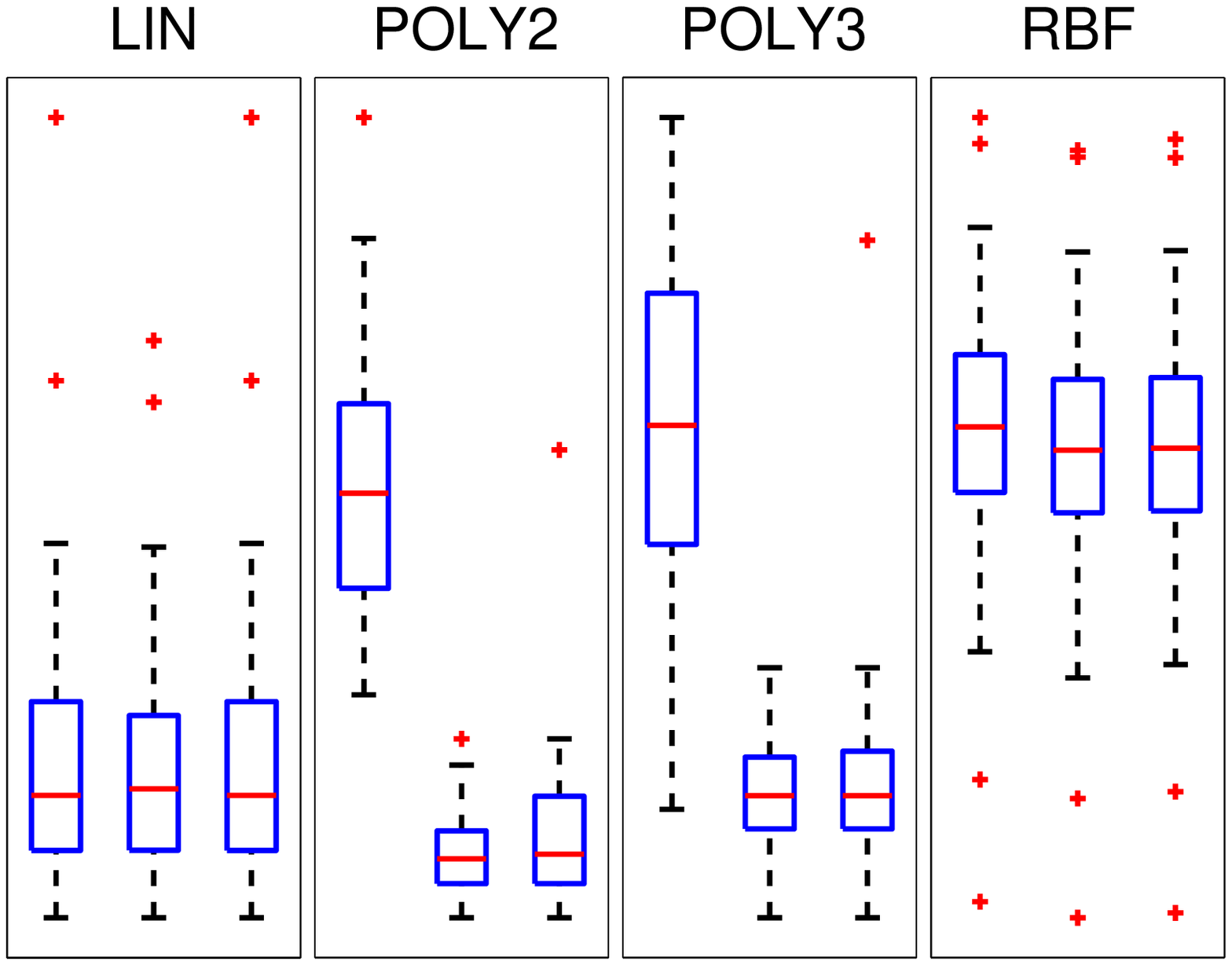}
    \label{fig:synthetic-sf1}
  } \hfill
  \subfigure[$\lambda = 0.1\times\gamma_0$]{
    \includegraphics[width=0.45\linewidth]{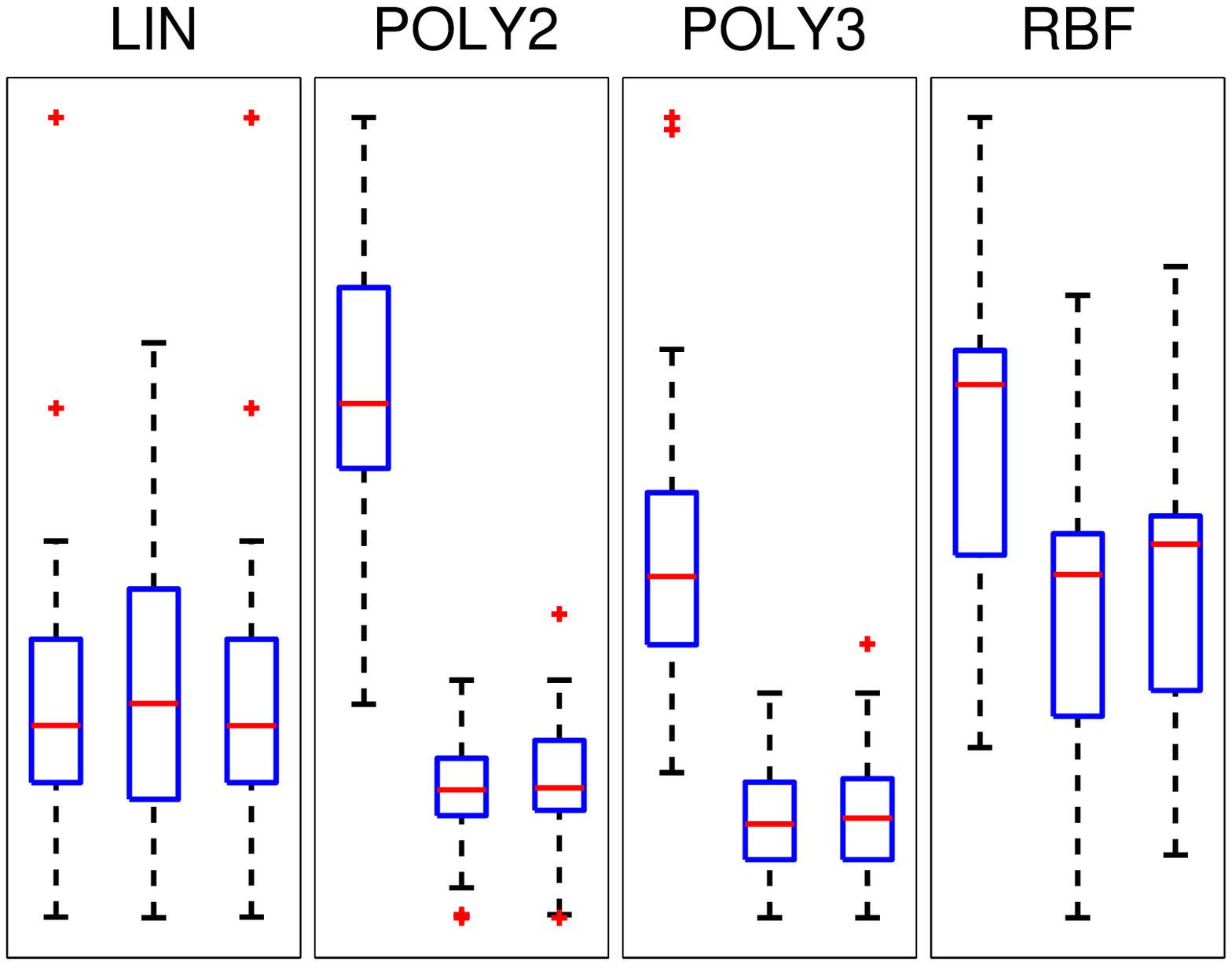}
    \label{fig:synthetic-sf2}
  }
  \subfigure[$\lambda = 1\times\gamma_0$]{
    \includegraphics[width=0.45\linewidth]{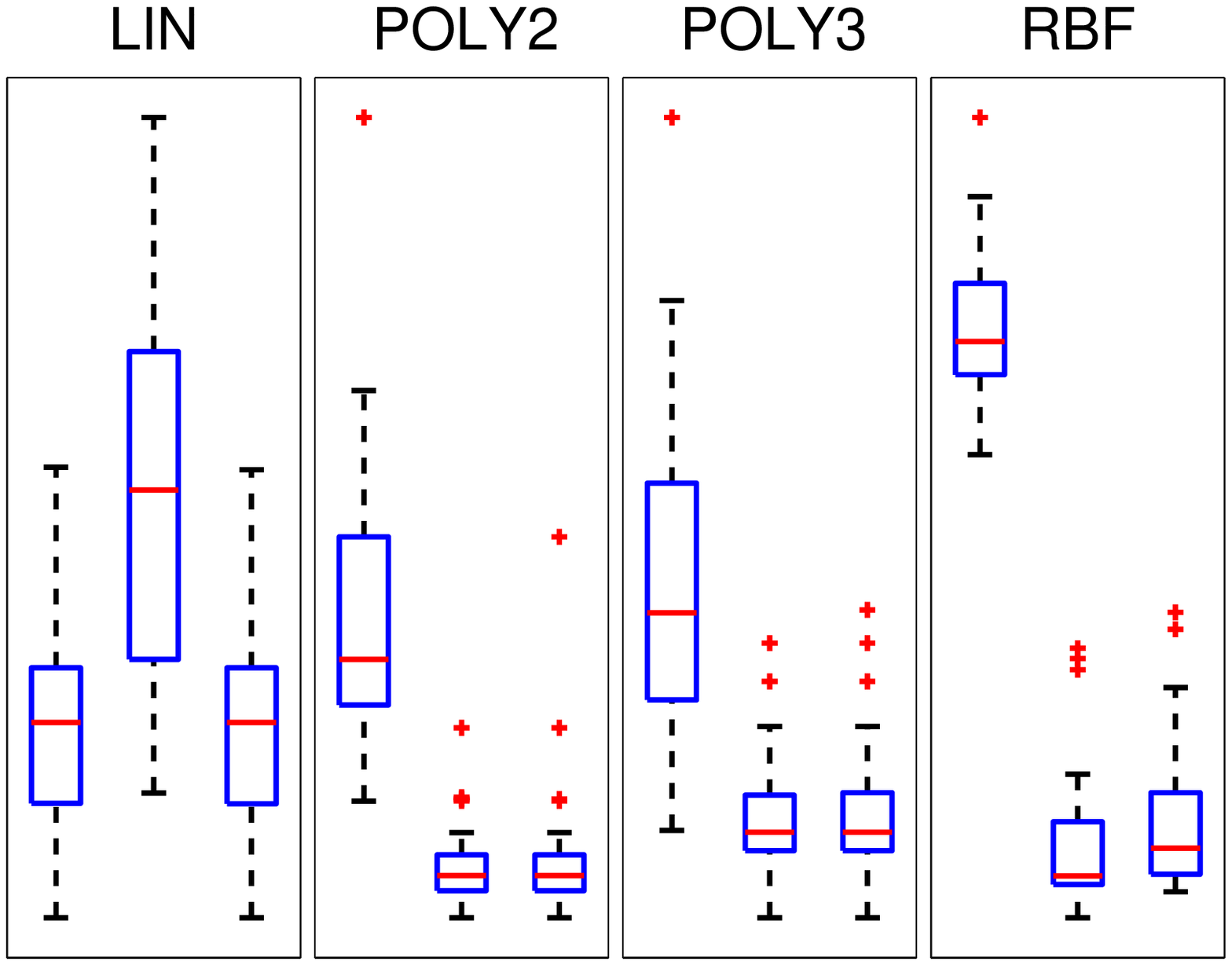}
    \label{fig:synthetic-sf3}
  } \hfill
  \subfigure[$\lambda = 10\times\gamma_0$]{
    \includegraphics[width=0.45\linewidth]{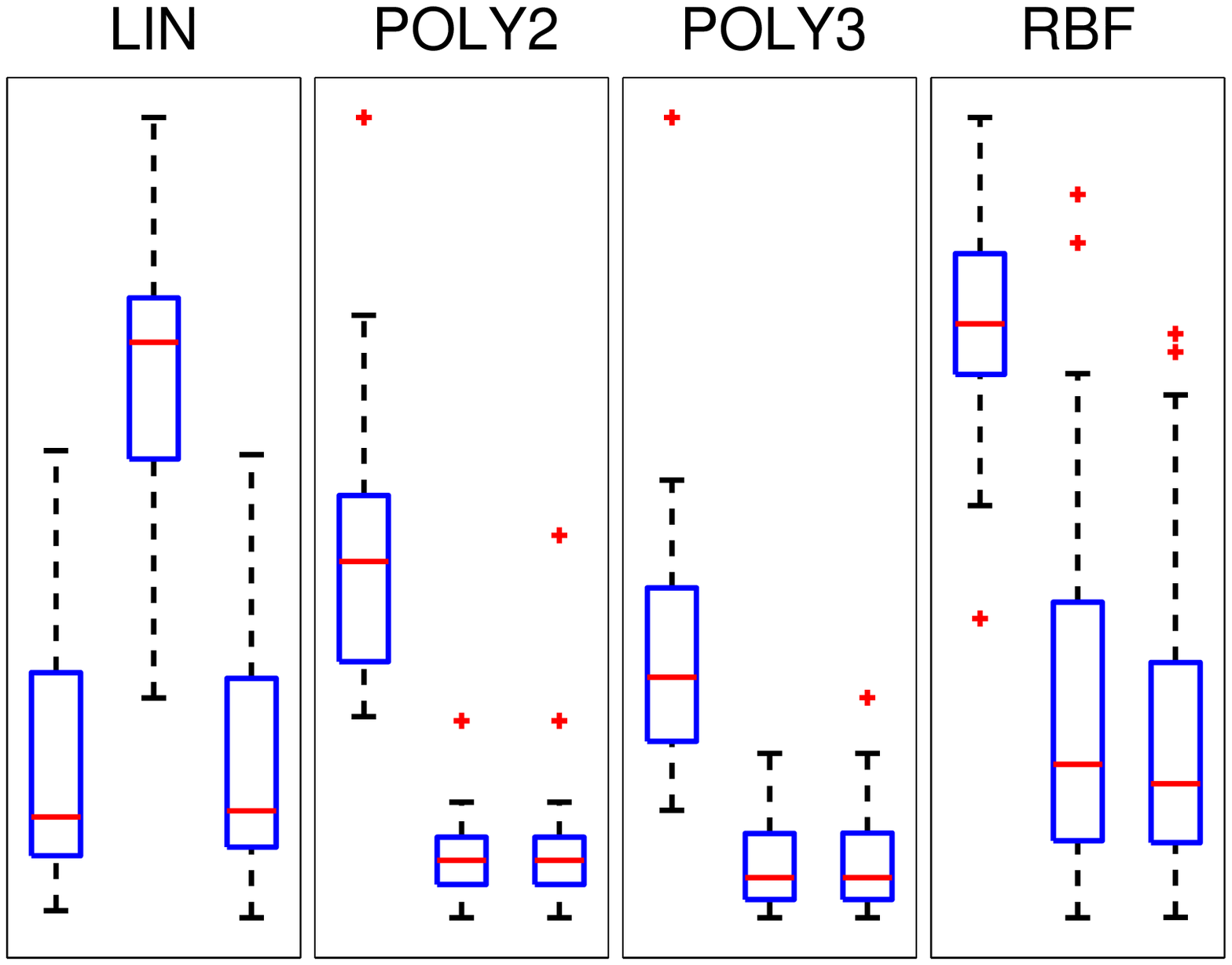}
    \label{fig:synthetic-sf4}
  }
  \caption{The average loss of KME, S-KMSE, and F-KMSE estimators with
    different values of shrinkage parameter. We repeat the experiments over 30
    different distributions with $n=10$ and $d=30$.}
  \label{fig:sim-results1}
\end{figure} 
  
\begin{figure}[t!] 
  \centering
  \includegraphics[width=\linewidth]{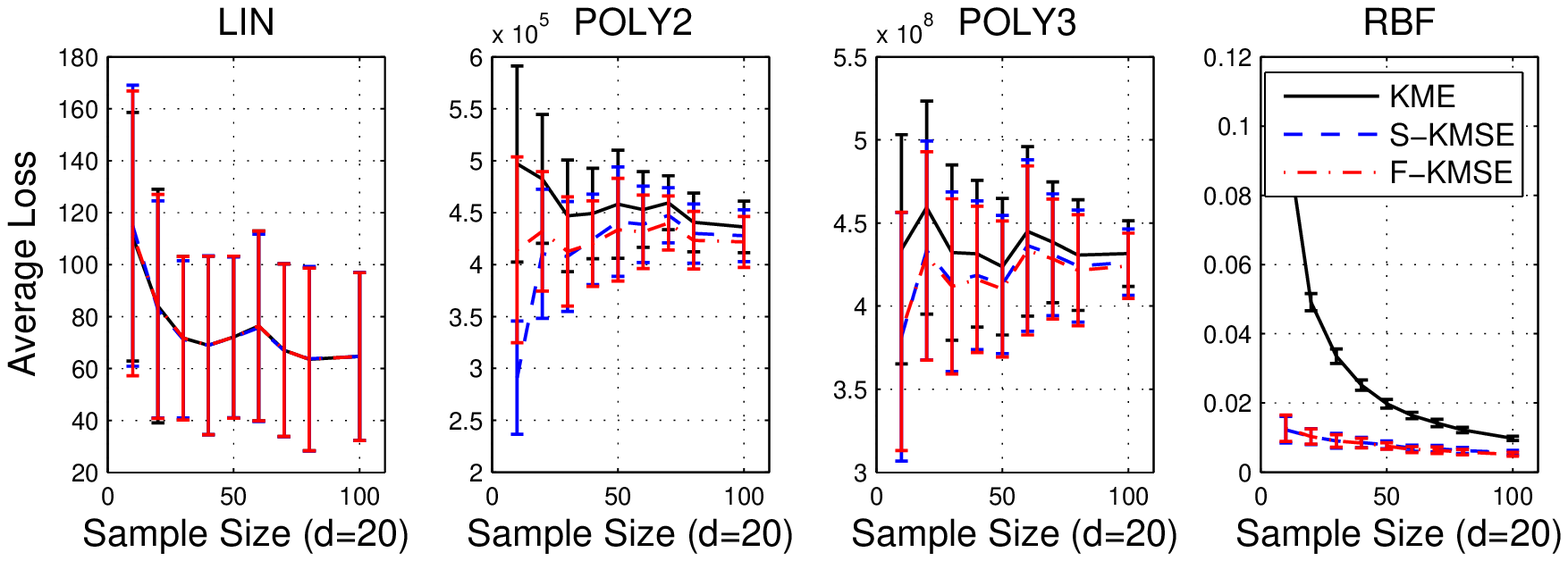}
  \includegraphics[width=\linewidth]{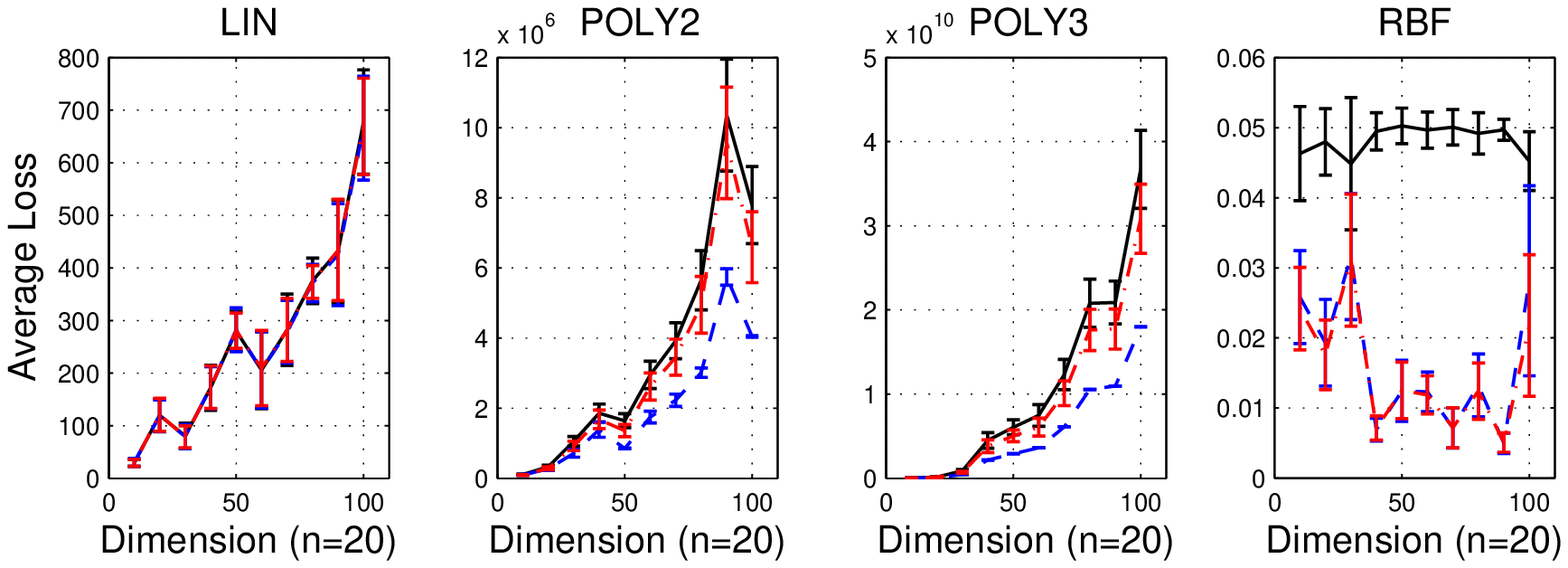}
  \caption{The average loss over 30 different distributions of KME, S-KMSE, and F-KMSE with varying sample size ($n$) and dimension ($d$). The shrinkage parameter $\lambda$ is chosen by an automatic LOOCV.}
  \label{fig:sim-results2}
\end{figure}

We evaluate different estimators using the loss function 
\begin{equation*}
\ell(\bvec) := \left\|\sum_{i=1}^n\beta_i k(x_i,\cdot) - \ep_{\pp{P}}[k(x,\cdot)]\right\|^2_{\hbspace} ,
\end{equation*}
\noindent where $\pp{P}$ is the true data-generating distribution and $\bvec$ is the weight vector associated with different estimators. To allow for an exact calculation of $\ell(\bvec)$, we consider when $\pp{P}$ is a mixture-of-Gaussians distribution and $k$ is the following kernel function: 1) linear kernel $k(x,x')=x^\top x'$; 2) polynomial degree-2 kernel $k(x,x')=(x^\top x' + 1)^2$; 3) polynomial degree-3 kernel $k(x,x')=(x^\top x' + 1)^3$; and 4) Gaussian RBF kernel $k(x,x') = \exp\left(-\|x-x'\|^2/2\sigma^2\right)$. In the following, we will refer to them as LIN, POLY2, POLY3, and RBF, respectively.

\subsubsection{Experimental protocol}
  Data are generated from a $d$-dimensional mixture of Gaussians:
  \begin{equation*}
    x \sim \sum_{i=1}^4\pi_i\mathcal{N}(\bm{\theta}_i,\Sigma_i) + \varepsilon,\;
    \theta_{ij} \sim  \mathcal{U}(-10,10), \;
    \Sigma_i \sim \mathcal{W}(2\times\mathbf{I}_d,7), \; 
    \varepsilon \sim \mathcal{N}(0,0.2\times\mathbf{I}_d),    
  \end{equation*}
  \noindent where $\mathcal{U}(a,b)$ and $\mathcal{W}(\Sigma_0,df)$ represent the uniform distribution and Wishart distribution, respectively. We set $\bm{\pi} = [0.05, 0.3, 0.4, 0.25]$. The choice of parameters here is quite arbritary; we have experimented using various parameter settings and the results are similar to those presented here. For the Gaussian RBF kernel, we set the bandwidth parameter to square-root of the median Euclidean distance between samples in the dataset (i.e., $\sigma^2 = \mathrm{median}\left\{\|x_i-x_j\|^2\right\}$ throughout).

Figure \ref{fig:sim-results1} shows the average loss of different estimators using different kernels as we increase the value of shrinkage parameter. Here we scale the shrinkage parameter by the minimum non-zero eigenvalue $\gamma_0$ of kernel matrix $\kmat$. In general, we find S-KMSE and F-KMSE tend to outperform KME. However, as $\lambda$ becomes large, there are some cases where shrinkage deteriorates the estimation performance, e.g., see LIN kernel and some outliers in Figure \ref{fig:sim-results1} when $\lambda$ is large. This suggests that it is very important to choose the parameter $\lambda$ appropriately (cf. the discussion in \S\ref{sec:inadmissibility}).

Similarly, Figure \ref{fig:sim-results2} depicts the average loss as we vary the sample size and dimension of the data. In this case, the shrinkage parameter is chosen by the proposed leave-one-out cross-validation score. As we can see, both S-KMSE and F-KMSE outperform the standard KME. The S-KMSE performs slightly better than the F-KMSE. Moreover, the improvement is more substantial in the ``large $d$, small $n$'' paradigm. In the worst cases, the S-KMSE and F-KMSE perform as well as the KME.

\subsection{Real Data}


\begin{sidewaystable}
  \centering
  \caption{Average negative log-likelihood of the model $Q$ on test
    points over 10 randomizations. The boldface represents the result
    whose difference from the baseline, i.e., KME, is statistically
    significant.} 
  \resizebox*{0.9\textwidth}{!}{
  \begin{tabular}{|rl|ccc|ccc|ccc|ccc|} 
    \hline
    \multicolumn{2}{|c|}{\multirow{2}{*}{\textbf{Dataset}}} & \multicolumn{3}{c|}{\textbf{LIN}} & \multicolumn{3}{c|}{\textbf{POLY2}} & \multicolumn{3}{c|}{\textbf{POLY3}} & \multicolumn{3}{c|}{\textbf{RBF}} \\    
    && KME & S-KMSE & F-KMSE & KME & S-KMSE & F-KMSE & KME & S-KMSE & F-KMSE & KME & S-KMSE & F-KMSE \\
    \hline
    1. & ionosphere & 33.2440 & 33.0325 & 33.1436 & 53.1266 & 53.7067 & 50.8695 & 51.6800 & 49.9149 & \textbf{47.4461} & 40.8961 &\textbf{40.5578} & \textbf{39.6804} \\
    2. & sonar & 72.6630 & 72.8770 & 72.5015 & 120.3454 & \textbf{108.8246} &\textbf{109.9980} & 102.4499 & \textbf{90.3920} & 91.1547 & 71.3048 & 70.5721 & \textbf{70.5830} \\
    3. & australian & 18.3703 & \textbf{18.3341} & 18.3719 & 18.5928 & 18.6028 & 18.4987 & 41.1563 & \textbf{34.4303} & \textbf{34.5460} & 17.5138 & 17.5637 & 17.4026 \\
    4. & specft & 56.6138 & 55.7374 & \textbf{55.8667} & 67.3901 & 65.9662 & 65.2056 & 63.9273 & 63.5571 & \textbf{62.1480} & 57.5569 & \textbf{56.1386} & \textbf{55.5808} \\
    5. & wdbc & 30.9778 & 30.9266 & \textbf{30.4400} & 93.0541 & 91.5803 & \textbf{87.5265} & 58.8235 & 54.1237 & \textbf{50.3911} & 30.8227 & \textbf{30.5968} & \textbf{30.2646} \\
    6. & wine & 15.9225 & 15.8850 & 16.0431 & 24.2841 & 24.1325 & \textbf{23.5163} & 35.2069 & 32.9465 & 32.4702 & 17.1523 & 16.9177 & \textbf{16.6312} \\
    7. & satimage$^*$ & 19.6353 & 19.8721 & 19.7943 & 149.5986 & 143.2277 & 146.0648 & 52.7973 & 57.2482 & 45.8946 & 20.3306 & 20.5020 & \textbf{20.2226} \\
    8. & segment & 22.9131 & 22.8219 & \textbf{22.0696} & 61.2712 & 59.4387 & \textbf{54.8621} & 38.7226 & 38.6226 & 38.4217 & 17.6801 & \textbf{16.4149} & \textbf{15.6814} \\ 
    9. & vehicle & 16.4145 & 16.2888 & 16.3210 & 83.1597 & \textbf{79.7248} & \textbf{79.6679} & 70.4340 & 63.4322 & 48.0177 & 15.9256 & 15.8331 & \textbf{15.6516} \\
    10. & vowel & 12.4227 & 12.4219 & 12.4264 & 32.1389 & \textbf{28.0474} & \textbf{29.3492} & 25.8728 & \textbf{24.0684} & \textbf{23.9747} & 12.3976 & 12.3823 & 12.3677 \\ 
    11. & housing & 15.5249 & 15.1618 & 15.3176 & 39.9582 & 37.1360 & 32.1028 & 50.8481 & 49.0884 & 35.1366 & 14.5576 & 14.3810 & \textbf{13.9379} \\
    12. & bodyfat & 17.6426 & \textbf{17.0419} & 17.2152 & 44.3295 & 43.7959 & \textbf{42.3331} & 27.4339 & 25.6530 & 24.7955 & 16.2725 & \textbf{15.9170} & \textbf{15.8665} \\
    13. & abalone$^*$ & 4.3348 & 4.3274 & 4.3187 & 14.9166 & 14.4041 & 11.4431 & 20.6071 & 23.2487 & 23.6291 & 4.6928 & 4.6056 & 4.6017 \\
    14. & glass & 10.4078 & 10.4451 & \textbf{10.4067} & 33.3480 & 31.6110 & 30.5075 & 45.0801 & 34.9608 & \textbf{25.5677} & 8.6167 & 8.4992 & \textbf{8.2469} \\
    \hline
  \end{tabular}}
\label{tab:kmm}
\end{sidewaystable}

We consider three benchmark applications: density estimation via kernel mean matching \cite{Song08:TDE}, kernel PCA using shrinkage mean and covariance operator \cite{Scholkopf98:NCA}, and discriminative learning on distributions \cite{Muandet12:SMM,Muandet13:OCSMM}. For the first two tasks we employ 14 datasets from the UCI repositories. We use only real-valued features, each of which is normalized to have zero mean and unit variance.

\subsubsection{Density estimation}
  We perform density estimation via kernel mean matching \cite{Song08:TDE}. That is, we fit the density $Q=\sum_{j=1}^m\pi_j\mathcal{N}(\bm{\theta}_j,\sigma_j^2\id)$ to each dataset by minimizing $\|\widehat{\mu} - \mu_{Q}\|_{\hbspace}^2$ s.t. $\sum_{j=1}^m\pi_j = 1$. The kernel mean $\widehat{\mu}$ is obtained from the samples using different estimators, whereas $\mu_{Q}$ is the kernel mean embedding of the density $Q$. Unlike experiments in \cite{Song08:TDE}, our goal is to compare different estimators of $\mu_{\pp{P}}$ where $\pp{P}$ is the true data distribution. A better estimate of $\mu_{\pp{P}}$ should lead to better density estimation, as measured by the negative log-likelihood of $Q$ on the test set. We use 30\% of the dataset as a test set. We set $m=10$ for each dataset. The model is initialized by running 50 random initializations using the k-means algorithm and returning the best. We repeat the experiments 10 times and perform the paired sign test on the results at the 5\% significance level.\footnote{The paired sign test is a nonparametric test that can be used to examine whether or not two paired samples have the same distribution. In our case, we compare S-KMSE and F-KMSE against KME.}

The average negative log-likelihood of the model $Q$, optimized via different estimators, is reported in Table \ref{tab:kmm}. Clearly, both S-KMSE and F-KMSE consistently achieve smaller negative log-likelihood when compared to KME. There are however few cases in which KME outperforms the proposed estimators, especially when the dataset is relatively large, e.g., \texttt{satimage} and \texttt{abalone}. We suspect that in those cases the standard KME already provides an accurate estimate of the kernel mean. To get a better estimate, more effort is required to optimize for the shrinkage parameter.

\subsubsection{Kernel PCA}
In this experiment, we perform the KPCA using different estimates of the mean and covariance operators. We compare the reconstruction error $\mathcal{E}_{proj}(z) = \|\phi(z) - \mathbf{P}\phi(z)\|^2$ on test samples where $\mathbf{P}$ is the projection constructed from the first 20 principal components. We use a Gaussian RBF kernel for all datasets. We compare 5 different scenarios: 1) standard KPCA; 2) shrinkage centering with S-KMSE; 3) shrinkage centering with F-KMSE; 4) KPCA with S-COSE; and 5) KPCA with F-COSE. To perform KPCA on shrinkage covariance operator, we solve the generalized eigenvalue problem $\kmat^c\mathbf{B}\kmat^c\mathbf{V} = \kmat^c\mathbf{VD}$ where $\mathbf{B}=\mathrm{diag}(\bvec)$ and $\kmat^c$ is the centered gram matrix. The weight vector $\bvec$ is obtained from our shrinkage estimators using the kernel matrix $\kmat^c\circ\kmat^c$ where $\circ$ denotes the Hadamard product. We use 30\% of the dataset as a test set.

\begin{sidewaysfigure}
\centering
\includegraphics[width=0.9\linewidth]{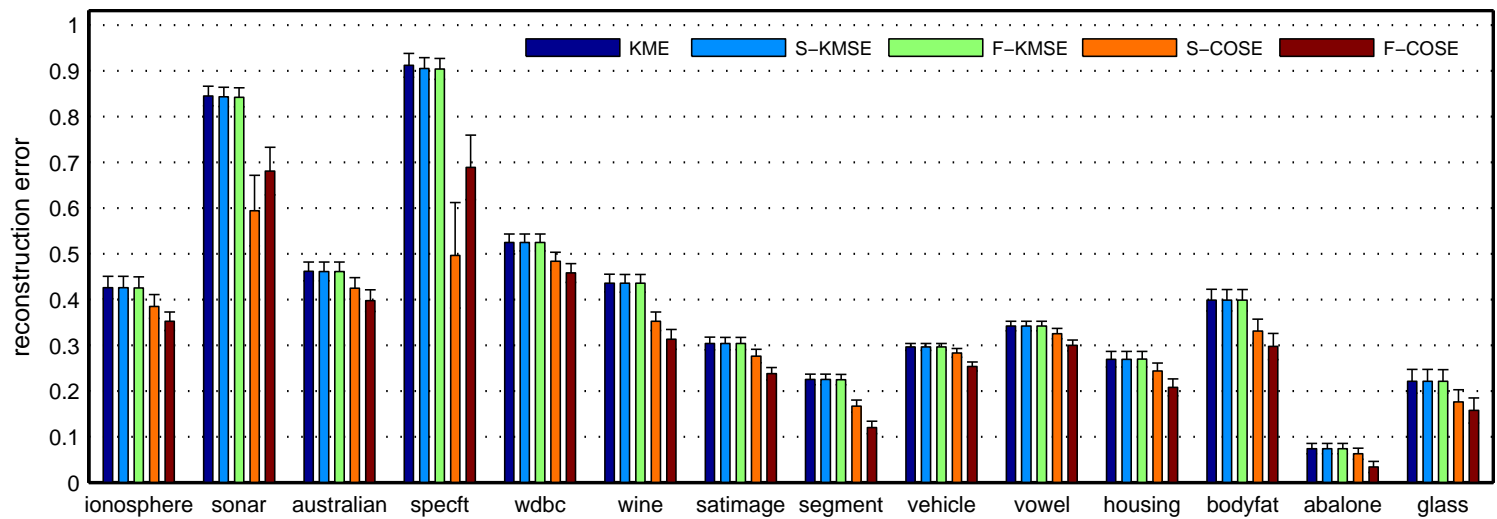}
\caption{The average reconstruction error of KPCA on hold-out test samples over 10 repetitions. The KME represents the standard approach, whereas S-KMSE and F-KMSE use shrinkage means to perform centering. The S-COSE and F-COSE directly use the shrinkage estimate of the covariance operator.}
\label{fig:real-results1}
\end{sidewaysfigure}

Figure \ref{fig:real-results1} illustrates the results of KPCA. Clearly, the S-COSE and F-COSE consistently outperforms all other estimators. Although we observe an improvement of S-KMSE and F-KMSE over KME, it is very small compared to that of S-COSE and F-COSE. This makes sense intuitively, since changing the mean point or shifting data does not change the covariance structure considerably, so it will not significantly effect the reconstruction error.

\subsubsection{Discriminative learning on distributions}
  A positive semi-definite kernel between distributions can be defined via their kernel mean embeddings. That is, given a training sample $(\widehat{\pp{P}}_1,y_1),\ldots,(\widehat{\pp{P}}_m,y_m)\in\mathscr{P}\times\{-1,+1\}$ where $\widehat{\pp{P}}_i := \frac{1}{n}\sum_{k=1}^{n}\delta_{x^{i}_{k}}$ and $x^{i}_{k}\sim\pp{P}_i$, the linear kernel between two distributions is approximated by $\langle\widehat{\mu}_{\pp{P}_i},\widehat{\mu}_{\pp{P}_j}\rangle = \langle \sum_{k=1}^n\beta^i_k \phi(x^i_k),\sum_{l=1}^n\beta^j_l \phi(x^j_l)\rangle = \sum_{k,l=1}^n\beta^i_k\beta^j_l k(x^{i}_{k},x^{j}_{l})$. The weight vectors $\bvec^i$ and $\bvec^j$ come from the kernel mean estimates of $\mu_{\pp{P}_i}$ and $\mu_{\pp{P}_j}$, respectively. The non-linear kernel can then be defined accordingly, e.g., $\kappa(\pp{P}_i,\pp{P}_j) = \exp(\|\widehat{\mu}_{\pp{P}_i} - \widehat{\mu}_{\pp{P}_j}\|^2_{\hbspace}/2\sigma^2)$. Our goal in this experiment is to investigate if the shrinkage estimate of the kernel mean improves the performance of the discriminative learning on distributions. To this end, we conduct experiments on natural scene categorization using support measure machine (SMM) \cite{Muandet12:SMM} and group anomaly detection on a high-energy physics dataset using one-class SMM (OCSMM) \cite{Muandet13:OCSMM}. We use both linear and non-linear kernels where the Gaussian RBF kernel is employed as an embedding kernel \cite{Muandet12:SMM}. All hyper-parameters are chosen by 10-fold cross-validation. For our unsupervised problem, we repeat the experiments using several parameter settings and report the best results.

\begin{table}[t!]
  \centering
    \caption{The classification accuracy of SMM and the area under ROC curve (AUC) of OCSMM using different kernel mean estimators to construct the kernel on distributions.}
    \label{tab:smm-ocsmm}
    \begin{tabular}{|l|cc|cc|}
      \hline
      \multirow{2}{*}{\textbf{Estimator}} & \multicolumn{2}{c|}{\textbf{Linear}} & \multicolumn{2}{c|}{\textbf{Non-linear}} \\
      & SMM & OCSMM & SMM & OCSMM \\
      \hline
      KME & 0.5432 & 0.6955 & 0.6017 & 0.9085 \\
      S-KMSE & 0.5521 & 0.6970 & 0.6303 & 0.9105 \\
      F-KMSE & 0.5610 & 0.6970 & 0.6522 & 0.9095 \\
      \hline
    \end{tabular}
  \end{table}

Table \ref{tab:smm-ocsmm} reports the classification accuracy of SMM and the area under ROC curve (AUC) of OCSMM using different kernel mean estimators. Clearly, both shrinkage estimators lead to better performance on both SMM and OCSMM when compared to KME.

To summarize, we find sufficient evidence to conclude that both S-KMSE and F-KMSE outperforms the standard KME, especially when the dataset is small. The performance of S-KMSE and F-KMSE is very competitive. The difference depends on the dataset and the kernel function.


\section{Conclusions}

To conclude, we show that the commonly used kernel mean estimator can be improved upon via Stein's phenomenon. Our theoretical result suggests that there exists a wide class of kernel mean estimators that are better than the standard one. To demonstrate this, we focus on two efficient shrinkage estimators, namely, simple and flexible kernel mean shrinkage estimators. Empirical study clearly shows that the proposed estimators outperform the standard one in various scenarios, especially in a ``large $d$, small $n$'' paradigm. Most importantly, the shrinkage estimates not only provide more accurate estimation, but also lead to superior performance on real-world applications. 


\bibliographystyle{abbrv-unsrt}
\bibliography{kmse-arxiv}

\newpage
\appendix

\section{James-Stein's Estimator}

Stein's result has transformed common belief in statistical world that
the maximum likelihood estimator, which is in common use for more than
a century, is optimal. Charles Stein showed in 1955 that it is
possible to uniformly improve the maximum likelihood estimator (MLE)
for the Gaussian model in terms of total squared error risk when
several parameters are estimated simultaneously from independent
normal observations \cite{Stein55:Inadmissible}. James and Stein later proposed a particularly
simple estimator which dominates the usual MLE, given that there are
more than two parameters \cite{Stein61:JSE}.

The following proposition gives a general form of the James-Stein's estimator.

\begin{proposition}
  Assuming $X\sim\mathcal{N}(\bm{\theta},\sigma^2\id)$ with $dim(X) \geq 3$,
  the estimator $\delta(X)=X$ for $\bm{\theta}$ is inadmissible under
  the squared loss function and is dominated by the following
  estimator
  \begin{equation*}
    \delta_{JS}(X) = \left(1 - \frac{(d-2)\sigma^2}{\|X\|^2}\right)X
  \end{equation*}
  \noindent where $d$ is the dimension of $X$.
\end{proposition}


Although the original works on James-Stein's estimator were entirely
written from the frequentist point of view, it was shown later that
James-Stein's estimator can be understood as an empirical Bayes
estimator \cite{Efron73:Bayes}. This is a treatment of James-Stein's
estimator from the Bayesian point of view. There have been a
considerable number of works in this direction, e.g.,
\cite{Efron72:Bayes,Efron73b:Bayes,Efron75:Bayes} and later by
\cite{Berger75:Minimax}, \cite{Bock75:Minimax},
\cite{Hudson78:Bayes}. Whether the same Bayesian interpretation is possible in an infinite-dimensional space such as the RKHS is still an
open problem.

The James-Stein's estimator is a special case of a larger class of estimators known as \emph{shrinkage estimator} \cite{Gruber98:Shrinkage}. In its most general form, the shrinkage estimator averages two different models: a high-dimensional model with low bias and high variance, and a lower dimensional model with larger bias but smaller variance. For example, one might consider the following estimator:
\begin{equation*}
  \hat{\theta}_{shrink} = \lambda\tilde{\theta} + (1-\lambda)\hat{\theta}_{ML}
\end{equation*}
\noindent where $\lambda \in [0,1]$, $\hat{\theta}_{ML}$ denotes
the usual maximum likelihood estimate of $\theta$, and
$\tilde{\theta}$ is an arbitrary point in the input space. In the case of
James-Stein's estimator, we have $\tilde{\theta}=\mathbf{0}$. That is,
it shrinks the usual estimator toward zero.

\section{Kernel Mean Shrinkage Estimator}

We give a detailed derivation of both simple kernel mean shrinkage
estimator (S-KMSE) and flexible kernel mean shrinkage estimator
(F-KMSE). Firstly, note that the loss we define in Section
\ref{sec:inadmissibility} is given by
\begin{equation}
  \label{eq:loss1}
  \ell(\mu,g) := \|\mu-g\|_{\hbspace}^2 = \|\ep[\phi(x)] - g\|^2_{\hbspace} .
\end{equation}
By Jensen's inequality, we can upper bound \eqref{eq:loss1} by the
loss functional
\begin{equation}
  \label{eq:loss2}
  \|\ep[\phi(x)] - g\|^2_{\hbspace} \leq \ep\|\phi(x) -
  g\|^2_{\hbspace} =: \mathcal{E}(g).
\end{equation}
Both \eqref{eq:loss1} and \eqref{eq:loss2} have a minimum at the same
$g$. In this paper, we formulate the problem in term of the loss
functional \eqref{eq:loss2} as it simplifies the analysis of
leave-one-out cross-validation score.

Given an i.i.d. sample $x_1,x_2,\ldots,x_n$, the KMSE can be obtained by minimizing the following loss functional
\begin{equation}
  \label{eq:kmse-loss}
  \widehat{\mathcal{E}}_{\lambda}(g) := \frac{1}{2n}\sum_{i=1}^n\left\|\phi(x_i)
    - g\right\|^2_{\hbspace} + \lambda\Omega(\|g\|),
\end{equation}
Different choices of $\Omega$ lead to different estimators, as outlined below.

\subsection{Simple Shrinkage}

By representer theorem, the solution of \eqref{eq:kmse-loss} can be written as
$g=\sum_{i=1}^n\beta_i\phi(x_i)$ for some $\bvec\in\rr^n$. Moreover,
the S-KMSE uses $\Omega(g) = \|g\|^2_{\hbspace}$. Substituting both
$g=\sum_{i=1}^n\beta_i\phi(x_i)$ and $\Omega(\|g\|) = \|g\|^2$ into \eqref{eq:kmse-loss} yields
\begin{equation}
  \label{eq:skmse-loss}
  \widehat{\mathcal{E}}_{\lambda}(\bvec) = \frac{1}{2n}\sum_{i=1}^n\left\| \phi(x_i) -
    \sum_{j=1}^n\beta_j \phi(x_j)\right\|_{\hbspace}^2 +
  \frac{\lambda}{2}\left\|\sum_{j=1}^n\beta_j
    \phi(x_j)\right\|^2_{\hbspace} .
\end{equation}
We can write \eqref{eq:skmse-loss} in term of the kernel function as
\begin{eqnarray*}
  \widehat{\mathcal{E}}_{\lambda}(\bvec) &=&
  \frac{1}{2n}\sum_{i=1}^n\left[ k(x_i,x_i) - 2\sum_{j=1}^n\beta_j
    k(x_j,x_i) + \sum_{j=1}^n\sum_{k=1}^n\beta_j\beta_k
    k(x_j,x_k)\right] + \frac{\lambda}{2}\bvec^{\top}\kmat\bvec \\
  &=& \frac{1}{2n}\sum_{i=1}^n k(x_i,x_i) -
  \frac{1}{n}\sum_{i,j=1}^n\beta_j k(x_j,x_i) +
  \frac{1}{2n}\sum_{i,j,k=1}^n\beta_j\beta_k k(x_j,x_k) +
  \frac{\lambda}{2} \bvec^{\top}\kmat\bvec \\
  &=& \frac{1}{2n}\mathrm{trace}(\kmat) - \bvec^\top\kmat\mathbf{1}_n
  + \frac{1}{2}\bvec^\top\kmat\bvec +
  \frac{\lambda}{2} \bvec^{\top}\kmat\bvec \\
  &=& \frac{1}{2n}\mathrm{trace}(\kmat) - \bvec^\top\kmat\mathbf{1}_n
  + \frac{1}{2}\bvec^\top(\kmat + \lambda\kmat)\bvec
\end{eqnarray*}
Taking the derivative of $\widehat{\mathcal{E}}_{\lambda}(\bvec)$ w.r.t. the vector $\bvec$ and setting
it to zero yield the optimal weight vector
\begin{equation*}
  \bvec = \left(\frac{1}{1+\lambda}\right)\mathbf{1}_n.
\end{equation*}
Consequently, the shrinkage estimator of the kernel mean is given by
\begin{equation*}
  \widehat{\mu}_{\lambda} = \sum_{i=1}^n\beta_i\phi(x_i) 
= \left(\frac{1}{1+\lambda}\right)\widehat{\mu}
= (1-\alpha)\widehat{\mu}
\end{equation*}
\noindent where $\alpha := \lambda/(1+\lambda) < 1$ and $\widehat{\mu}$ denotes the standard kernel mean estimator.

\subsection{Flexible Shrinkage}

Similarly, the flexible KMSE is obtained by minimizing 
\begin{equation*}
 \widehat{\mathcal{E}}_{\lambda}(\bvec) = \frac{1}{2n}\sum_{i=1}^n\left\| \phi(x_i) - \sum_{j=1}^n\beta_j \phi(x_j)\right\|_{\hbspace}^2 + \frac{\lambda}{2}\bvec^{\top}\bvec
\end{equation*}
\noindent with respect to the weight vector $\bvec\in\rr^n$. It can be
rewritten in term of the kernel function as
\begin{eqnarray*}
 \widehat{\mathcal{E}}_{\lambda}(\bvec) &=& \frac{1}{2n}\mathrm{trace}(\kmat) - \bvec^\top\kmat\mathbf{1}_n
  + \frac{1}{2}\bvec^\top\kmat\bvec +
  \frac{\lambda}{2} \bvec^{\top}\bvec \\ 
  &=& \frac{1}{2n}\mathrm{trace}(\kmat) - \bvec^\top\kmat\mathbf{1}_n
  + \frac{1}{2}\bvec^\top(\kmat + \lambda\id)\bvec
\end{eqnarray*}
Taking the derivative of $\widehat{\mathcal{E}}_{\lambda}(\bvec)$ with respect to $\bvec$ and setting it to zero yield
\begin{align*}
  \frac{\partial \widehat{\mathcal{E}}_{\lambda}}{\partial\bvec} = 0 \Rightarrow  -\kmat\mathbf{1}_n + (\kmat + \lambda\mathbf{I})\bvec &= 0 \\
   (\kmat + \lambda\mathbf{I})\bvec &= \kmat\mathbf{1}_n \\
  \bvec &= (\kmat + \lambda\mathbf{I})^{-1}\kmat\mathbf{1}_n
\end{align*}

\noindent where $\mathbf{1}_n$ denotes an $n\times 1$ vector whose elements are all $1/n$.

\section{Proof of Theorem \ref{thm:akmse-loocv}}

In this section we adopt the approach similar to the one presented in \cite{Green94:NREG} for ridge regression problem. For a given shrinkage parameter $\lambda$, let us consider the observation $x_i$ as being a new observation by omitting it from the dataset. Denote by $\widehat{\mu}^{(-i)}_{\lambda} = \sum_{j\neq i}\beta_j^{(-i)}\phi(x_j)$ the kernel mean estimated from the remaining data, using the value $\lambda$ as a shrinkage parameter, so that $\bvec^{(-i)}$ is the minimizer of 
\begin{equation*}
  \sum_{j\neq i} \left\|\phi(x_j) - \sum_{k\neq
      i}\beta_k\phi(x_k)\right\|^2_{\hbspace} + \lambda\|\bvec\|^2 .
\end{equation*}
We will measure the quality of $\widehat{\mu}^{(-i)}_{\lambda}$ by how well it approximates $\phi(x_i)$. The overall quality of the estimate can be quantified by the cross-validation score function
\begin{equation*}
  LOOCV(\lambda) = \frac{1}{n}\sum_{i=1}^n\left\| \phi(x_i) -
    \widehat{\mu}_{\lambda}^{(-i)}\right\|^2_{\hbspace} .
\end{equation*}
Note that the vector $\bvec^{(-i)}$ has length $n-1$, whereas the
original vector $\bvec$ has length $n$. To simplify the following
analysis, we will assume that $\bvec^{(-i)}$ has length $n$ with
$\beta_i=0$. Note that this representation does not alter the
leave-one-out estimate $\widehat{\mu}^{(-i)}_{\lambda}$.

\begin{thm-hand}[\ref{thm:akmse-loocv}.]
  The LOOCV score of F-KMSE satisfies 
  \begin{equation*}
  LOOCV(\lambda) =
  \frac{1}{n}\sum_{i=1}^n (\bvec^{\top}\kmat - \kmat_{\cdot i})^\top
  \mathbf{C}_{\lambda}(\bvec^{\top}\kmat - \kmat_{\cdot i})
  \end{equation*}
  \noindent where $\bvec$ is the weight vector calculated from the full dataset with the shrinkage parameter $\lambda$ and $\mathbf{C}_{\lambda} = (\kmat - \frac{1}{n}\kmat(\kmat+\lambda\id)^{-1}\kmat)^{-1} \kmat (\kmat - \frac{1}{n}\kmat(\kmat + \lambda \id)^{-1}\kmat)^{-1}$.
\end{thm-hand}

Note that the leave-one-out cross-validation score in Theorem
\ref{thm:akmse-loocv} does not depend on the leave-one-out solution
$\bvec_{\lambda}^{(-i)}$, but depends only on the non-leave-one-out
solution $\bvec_{\lambda}$. Consequently, the overall score can be computed efficiently.

\begin{proof}[Proof of Thorem \ref{thm:akmse-loocv}]

To prove Theorem \ref{thm:akmse-loocv}, we first show that the
leave-one-out solution $\bvec_{\lambda}^{(-i)}$ can be obtained via
the standard formulation with modified target vector.

\begin{lemma}
  \label{lem:kmse-target}
  For fixed $\lambda$ and $i$, let $\bvec^{(-i)}$ denote the vector
  with components $\beta^{(-i)}_j$ for $j\neq i$. Let us define a vector $\Phi^* = [\phi(x_1),\ldots,\phi(x_{i-1}),\widehat{\mu}^{(-i)}_{\lambda},\phi(x_{i+1}),\ldots,\phi(x_n)]^{\top}$ and a matrix $\mathbf{B}^*_{ml} = \langle \phi(x_m), \Phi^*_l\rangle_{\hbspace}$. Then $\bvec^{(-i)} = (\kmat + \lambda \id)^{-1}\mathbf{B}^*\mathbf{1}_n$.
\end{lemma}

\begin{proof}
  For any vector $\bvec$,
  \begin{eqnarray*}
    \sum_{j=1}^n\left\|\Phi^*_j - \sum_{k=1}^n\beta_k\phi(x_k)\right\|^2_{\hbspace} + \lambda\|\bvec\|^2 &\geq&  \sum_{j\neq i}\left\|\Phi^*_j - \sum_{k=1}^n\beta_k\phi(x_k)\right\|^2_{\hbspace} + \lambda\|\bvec\|^2 \\
    &\geq& \sum_{j\neq i}\left\|\Phi^*_j - \sum_{k=1}^n\beta^{(-i)}_k\phi(x_k)\right\|^2_{\hbspace} + \lambda\|\bvec^{(-i)}\|^2 \\
    &=& \sum_{j=1}^n\left\| \Phi^*_j - \sum_{k=1}^n\beta_k^{(-i)}\phi(x_k) \right\|^2_{\hbspace} + \lambda\|\bvec^{(-i)}\|^2
  \end{eqnarray*}
  \noindent by the definition of $\bvec^{(-i)}$ and the fact that $\Phi^*_i = \widehat{\mu}^{(-i)}_{\lambda}$. It follows that $\bvec^{(-i)}$ is the minimizer of $\sum_j\|\Phi^*_j - \sum_k\beta_k\phi(x_k)\|^2_{\hbspace} + \lambda\|\bvec\|^2$ so that $\bvec^{(-i)} = (\kmat + \lambda\id)^{-1}\mathbf{B}^*\mathbf{1}_n$, as required.
\end{proof}

As we can see, the resulted formulation of $\bvec^{(-i)}$ in Lemma
\ref{lem:kmse-target} depends on the leave-one-out solution
$\widehat{\mu}^{(-i)}_{\lambda}$ which in turn requires a knowledge of
$\bvec^{(-i)}$. As a result, we cannot use this formulation to compute $\bvec^{(-i)}$ in
practice. However, it will be very useful as an
intermediate step in deriving
the leave-one-out cross-validation score.

In the following, we will write $\mathbf{A}$ for
$(\kmat+\lambda\mathbf{I})^{-1}$ throughout. By virtue of Lemma \ref{lem:kmse-target}, we can write an expression for the deleted residual $\phi(x_i) - \widehat{\mu}^{(-i)}_{\lambda}$ as
\begin{align*}
  \widehat{\mu}^{(-i)}_{\lambda} - \phi(x_i) &= \sum_{j=1}^n\beta^{(-i)}_j\phi(x_j) - \phi(x_i) \\
  &= \frac{1}{n}\sum_{j=1}^n\sum_{m=1}^n\left\{\mathbf{A}\mathbf{B}^*\right\}_{jm}\phi(x_j) - \phi(x_i) \\
  &= \frac{1}{n}\sum_{j=1}^n\sum_{m\neq i} \left\{\mathbf{A}\kmat\right\}_{jm}\phi(x_j) + \frac{1}{n}\sum_{j=1}^n\sum_{l=1}^n\mathbf{A}_{jl}\mathbf{B}^*_{li}\phi(x_j) - \phi(x_i) \\
  &= \frac{1}{n}\sum_{j=1}^n\sum_{m\neq i} \left\{\mathbf{A}\kmat\right\}_{jm}\phi(x_j) + \frac{1}{n}\sum_{j=1}^n\sum_{l=1}^n\mathbf{A}_{jl}\langle \phi(x_l),\widehat{\mu}^{(-i)}_{\lambda}\rangle\phi(x_j) - \phi(x_i) \\
  &= \frac{1}{n}\sum_{j=1}^n\sum_{m=1}^n \left\{\mathbf{A}\kmat\right\}_{jm}\phi(x_j) - \phi(x_i) \\ 
  & \quad - \frac{1}{n}\sum_{j=1}^n\left\{\mathbf{A}\kmat\right\}_{ji}\phi(x_j) + \frac{1}{n}\sum_{j=1}^n\sum_{l=1}^n\mathbf{A}_{jl}\langle \phi(x_l),\widehat{\mu}^{(-i)}_{\lambda}\rangle\phi(x_j) \\
  &= \frac{1}{n}\sum_{j=1}^n\sum_{m=1}^n \left\{\mathbf{A}\kmat\right\}_{jm}\phi(x_j) - \phi(x_i) \\ 
  & \quad - \frac{1}{n}\sum_{j=1}^n\sum_{l=1}^n\mathbf{A}_{jl}\langle\phi(x_l),\phi(x_i)\rangle\phi(x_j) + \frac{1}{n}\sum_{j=1}^n\sum_{l=1}^n\mathbf{A}_{jl}\langle \phi(x_l),\widehat{\mu}^{(-i)}_{\lambda}\rangle\phi(x_j) \\
  &= \frac{1}{n}\sum_{j=1}^n\sum_{m=1}^n \left\{\mathbf{A}\kmat\right\}_{jm}\phi(x_j) - \phi(x_i) + \frac{1}{n}\sum_{j=1}^n\sum_{l=1}^n\mathbf{A}_{jl}\langle \phi(x_l),\widehat{\mu}^{(-i)}_{\lambda}-\phi(x_i)\rangle\phi(x_j) \\
  &= \widehat{\mu}_{\lambda} - \phi(x_i) + \frac{1}{n}\sum_{j=1}^n\sum_{l=1}^n\mathbf{A}_{jl}\langle \phi(x_l),\widehat{\mu}^{(-i)}_{\lambda}-\phi(x_i)\rangle\phi(x_j)
\end{align*}

Denote the deleted residual $\widehat{\mu}^{(-i)}_{\lambda} - \phi(x_i)$ by $\Delta_{\lambda}^{(-i)}$. Then, the above equation can be rewritten as
\begin{equation}
  \label{eq:newnotation}
  \Delta_{\lambda}^{(-i)} = \widehat{\mu}_{\lambda} - \phi(x_i) + \frac{1}{n}\sum_{j=1}^n\sum_{l=1}^n\mathbf{A}_{jl}\langle \phi(x_l),\Delta_{\lambda}^{(-i)}\rangle\phi(x_j) .
\end{equation}
Since the deleted residual $\Delta_{\lambda}^{(-i)}$ lies in the
subspace spanned by the samples $\phi(x_1),\ldots,\phi(x_n)$, we may write 
\begin{equation*}
  \Delta_{\lambda}^{(-i)} = \sum_{k=1}^n\xi_k\phi(x_k)
\end{equation*}
\noindent for some $\bm{\xi}\in \rr^n$. Substituting back into \eqref{eq:newnotation} yields
\begin{eqnarray*}
  \sum_{k=1}^n\xi_k\phi(x_k) &=& \widehat{\mu}_{\lambda}-\phi(x_i) + \frac{1}{n}\sum_{j,l}\mathbf{A}_{jl}\langle \phi(x_l),\sum_{k=1}^n\xi_k\phi(x_k)\rangle\phi(x_j) \\
  &=& \widehat{\mu}_{\lambda}-\phi(x_i) + \frac{1}{n}\sum_{j,l}\mathbf{A}_{jl}\sum_{k=1}^n\xi_k\langle\phi(x_l),\phi(x_k)\rangle\phi(x_j) \\
  &=& \widehat{\mu}_{\lambda}-\phi(x_i) + \frac{1}{n}\sum_{j,l}\mathbf{A}_{jl}\sum_{k=1}^n\xi_k\kmat_{lk}\phi(x_j) \\
  &=& \widehat{\mu}_{\lambda}-\phi(x_i) + \frac{1}{n}\sum_{j=1}^n\sum_{k=1}^n \sum_{l=1}^n \mathbf{A}_{jl}\kmat_{lk}\xi_k\phi(x_j) \\
  &=& \widehat{\mu}_{\lambda}-\phi(x_i) + \frac{1}{n}\sum_{j=1}^n\sum_{k=1}^n \left\{\mathbf{A}\kmat\right\}_{jk}\xi_k\phi(x_j) \\
  &=& \widehat{\mu}_{\lambda}-\phi(x_i) + \frac{1}{n}\sum_{j=1}^n \left\{\mathbf{A}\kmat\bm{\xi}\right\}_j\phi(x_j)
\end{eqnarray*}
By taking the inner product on both sides of the equation with respect
to the samples $\phi(x_1),\ldots,\phi(x_n)$, the optimal $\bm{\xi}$
can be obtained by solving the system of equations:
\begin{eqnarray*}
  \kmat\bm{\xi} &=& \bvec^\top\kmat - \kmat_{\cdot i} + \frac{1}{n}\kmat\mathbf{A}\kmat\bm{\xi} \\
  (\kmat - \frac{1}{n}\kmat\mathbf{A}\kmat)\bm{\xi} &=&
  \bvec^\top\kmat - \kmat_{\cdot i} \\
  \bm{\xi} &=& (\kmat -
  \frac{1}{n}\kmat\mathbf{A}\kmat)^{-1}(\bvec^\top\kmat - \kmat_{\cdot
    i}),
\end{eqnarray*}
\noindent where $\kmat_{\cdot i}$ denotes the $i$th column of matrix $\kmat$. Consequently, the leave-one-out cross-validation score for the sample $x_i$ can be computed by
\begin{eqnarray*}
  \left\|\Delta_{\lambda}^{(-i)}\right\|^2_{\hbspace} &=&
  \bm{\xi}^\top\kmat\bm{\xi} \\
  &=& (\bvec^\top\kmat - \kmat_{\cdot
    i})^\top (\kmat - \frac{1}{n}\kmat\mathbf{A}\kmat)^{-1} \kmat
  (\kmat - \frac{1}{n}\kmat\mathbf{A}\kmat)^{-1}(\bvec^\top\kmat -
  \kmat_{\cdot i}) \\
  &=& (\bvec^\top\kmat - \kmat_{\cdot i})^\top
  \mathbf{C}_{\lambda}(\bvec^\top\kmat - \kmat_{\cdot i})
\end{eqnarray*}
\noindent where $\mathbf{C}_{\lambda} = (\kmat - \frac{1}{n}\kmat\mathbf{A}\kmat)^{-1} \kmat
  (\kmat - \frac{1}{n}\kmat\mathbf{A}\kmat)^{-1}$. Hence, we have the score over full dataset 
\begin{equation*}
  LOOCV(\lambda) = \frac{1}{n}\sum_{i=1}^n
  \left\|\Delta_{\lambda}^{(-i)}\right\|^2_{\hbspace} =
  \frac{1}{n}\sum_{i=1}^n (\bvec^\top\kmat - \kmat_{\cdot i})^\top
  \mathbf{C}_{\lambda}(\bvec^\top\kmat - \kmat_{\cdot i})
\end{equation*}
as required.
\end{proof}

\section{Shrinkage Centering in Feature Space}

In many applications of kernel methods, it is often assumed that the kernel is centered. That is, the feature map of the data in feature space is given by
\begin{equation*}
  \tilde{\phi}(x) = \phi(x) - \ep[\phi(x)] .
\end{equation*}
In practice, the feature mean $\ep[\phi(X)]$ is approximated using the empirical average $\frac{1}{n}\sum_{i=1}^n\phi(x_i)$ such that the centered feature map can be written as
\begin{equation*}
  \tilde{\phi}(x) = \phi(x) - \frac{1}{n}\sum_{i=1}^n\phi(x_i) .
\end{equation*}
However, it is very difficult to explicitly center the data because
the feature space can be high-dimensional, if not infinite. In
\cite{Scholkopf98:NCA}, it is shown that we can compute the centered kernel in terms of the non-centered kernel alone. 

A direct application of our shrinkage estimators is to replace the empirical average in the above formulation by its shrinkage version, i.e., 
\begin{equation*}
  \tilde{\phi}(x) = \phi(x) - \sum_{i=1}^n\beta_i\phi(x_i)
\end{equation*}
and thereby the centered kernel $\kmat^c$ can be written as
\begin{eqnarray*}
  \kmat^c_{ij} &=& \left(\phi(x_i) - \sum_{k=1}^n\beta_k\phi(x_k)\right) ^{\top}\left(\phi(x_j) - \sum_{k=1}^n\beta_k\phi(x_k)\right) \\
  &=& \phi(x_i) ^\top\phi(x_j) - \phi(x_i)^\top\left[\sum_{k=1}^n\beta_k\phi(x_k)\right] - \left[\sum_{l=1}^n\beta_l\phi(x_l)^\top\right]\phi(x_j) \\
  && + \left[\sum_{k=1}^n\beta_k\phi(x_k)^\top\right]\left[\sum_{l=1}^n\beta_l\phi(x_l)\right] \\
  &=& \kmat_{ij} - \bvec^\top\kmat_{\cdot i} - \kmat_{\cdot j}^\top\bvec + \bvec^\top\kmat\bvec,
\end{eqnarray*}
\noindent where $\bvec$ is obtained from the shrinkage
estimators. Defining an $n\times n$ matrix $\mathbf{B} =
[\bvec,\bvec,\ldots,\bvec]$, we can write a compact expression of
centering operation as
\begin{equation*}
  \kmat^c = \kmat - \mathbf{B}^\top\kmat - \kmat\mathbf{B} + \mathbf{B}^{\top}\kmat\mathbf{B}.
\end{equation*}

Consider a set of test points $x^*_1,x^*_2,\ldots,x^*_m$ and define an $m\times n$ test kernel matrix by 
\begin{equation*}
  \mathbf{L}_{ij} = \langle \phi(x^*_i),\phi(x_j)\rangle_{\hbspace}.
\end{equation*}
Thus, the centered test kernel matrix can be similarly obtained as
\begin{equation*}
  \mathbf{L}^c = \mathbf{L} - \mathbf{B}_t\kmat - \mathbf{L}\mathbf{B} + \mathbf{B}_t\kmat\mathbf{B}
\end{equation*}
\noindent where $\mathbf{B}_t = [\bvec,\bvec,\ldots,\bvec]^{\top}$ denotes an
$m\times n$ matrix.

\section{Covariance-operator Shrinkage Estimator}

We can extend the idea to improving the estimation of cross-covariance operator on the RKHS. It is a foundation to several kernel-based approaches such as kernel PCA, kernel Fisher discriminant analysis, and kernel CCA. The covariance operator can be seen as a mean function in the joint space.

Let $(\hbspace_X,k_X)$ and $(\hbspace_Y,k_Y)$ be the RKHS of functions on measurable space $\inspace$ and $\mathcal{Y}$, respectively, with positive definite kernel $k_X$ and $k_Y$ (with feature map $\phi$ and $\varphi$). In this section, we will consider a random vector $(X,Y):\Omega \rightarrow \inspace\times\mathcal{Y}$ with distribution $\pp{P}_{XY}$. The marginal distributions of $X$ and $Y$ are denoted by $\pp{P}_X$ and $\pp{P}_Y$, respectively. We assume that $\ep_X[k_X(X,X)] < \infty$ and $\ep_Y[k_Y(Y,Y)] < \infty$.

One can show that there exists a unique cross-covariance operator $\Sigma_{YX}:\hbspace_X\rightarrow\hbspace_Y$ such that
\begin{equation*}
  \langle g,\Sigma_{YX}f\rangle_{\hbspace_Y} = \ep_{XY}[(f(X) - \ep_X[f(X)])(g(Y) - \ep_Y[g(Y)])] = Cov(f(X),g(Y))
\end{equation*}
holds for all $f\in\hbspace_X$ and $g\in\hbspace_Y$. If $X$ is equal to $Y$, we obtain the self-adjoint operator $\Sigma_{XX}$ called the covariance operator.

Given an i.i.d sample from $\pp{P}_{XY}$ written as $(x_1,y_1),(x_2,y_2),\ldots,(x_n,y_n)$, we can write the empirical cross-covariance operator $\widehat{\Sigma}_{YX}$ as
\begin{equation}
  \label{eq:emp-cco}
  \widehat{\Sigma}_{YX} := \frac{1}{n}\sum_{i=1}^n\phi(x_i)\otimes\varphi(y_i) - \widehat{\mu}_X\otimes\widehat{\mu}_Y
\end{equation}
\noindent where $\widehat{\mu}_X = \frac{1}{n}\sum_{i=1}^n\phi(x_i)$ and $\widehat{\mu}_Y = \frac{1}{n}\sum_{i=1}^n\varphi(y_i)$. Let assume that $\tilde{\phi}$ and $\tilde{\varphi}$ are the centered version of the feature map $\phi$ and $\varphi$, respectively. Then, the empirical cross-covariance operator \eqref{eq:emp-cco} can be rewritten as
\begin{equation*}
  \widehat{\Sigma}_{YX} := \frac{1}{n}\sum_{i=1}^n\tilde{\phi}(x_i)\otimes\tilde{\varphi}(y_i) ,
\end{equation*}
\noindent which can be obtained as a minimizer of the following loss functional:
\begin{equation}
  \label{eq:cose-loss}
  \widehat{\mathcal{E}}(g) := \frac{1}{n}\sum_{i=1}^n\left\| \tilde{\phi}(x_i)\otimes\tilde{\varphi}(y_i) - g\right\|^2_{\hbspace_X\otimes\hbspace_Y}, \quad g\in\hbspace_X\otimes\hbspace_Y \,.
\end{equation}

Assume that $g$ lies in the subspace spanned by the data, i.e., $g=\sum_{i=1}^n\beta_i\tilde{\phi}(x_i)\otimes\tilde{\varphi}(y_i)$. By the inner product property in product space, we have $\langle \tilde{\phi}(x)\otimes\tilde{\varphi}(y),\tilde{\phi}(x')\otimes\tilde{\varphi}(y')\rangle_{\hbspace_X\otimes\hbspace_Y} = \langle\tilde{\phi}(x),\tilde{\phi}(x')\rangle_{\hbspace_X}\langle\tilde{\varphi(y)},\tilde{\varphi(y')}\rangle_{\hbspace_Y} = k_X(x,x')k_Y(y,y')$.

Note that \eqref{eq:cose-loss} is of the same form as the kernel mean
estimator. As a result, we can apply the same analysis throughout.

\end{document}